\documentclass{article}

\PassOptionsToPackage{numbers, compress}{natbib}

\usepackage[preprint]{neurips_2021}

\usepackage[utf8]{inputenc} %
\usepackage[T1]{fontenc}    %
\usepackage[colorlinks,
    linkcolor=red,citecolor=citecolor,urlcolor=blue]{hyperref}       %
\usepackage{url}            %
\usepackage{booktabs}       %
\usepackage{amsfonts}       %
\usepackage{nicefrac}       %
\usepackage{microtype}      %
\usepackage{xcolor}         %

\usepackage{algorithm}
\usepackage{graphicx}
\usepackage{threeparttable}
\usepackage{multirow}
\usepackage{amsmath}
\usepackage{bm}
\usepackage{bbm}
\usepackage{amsthm}
\usepackage[]{algpseudocode}
\usepackage{enumitem} %
\usepackage[subrefformat=parens,labelformat=parens]{subfig}

\usepackage{wrapfig}
\usepackage[algo2e, ruled, vlined, linesnumbered, noend]{algorithm2e}

\definecolor{citecolor}{RGB}{34,139,34}
\definecolor{mydarkblue}{rgb}{0,0.08,1}
\definecolor{mydarkgreen}{rgb}{0.02,0.6,0.02}
\definecolor{mydarkred}{rgb}{0.8,0.02,0.02}
\definecolor{mydarkorange}{rgb}{0.40,0.2,0.02}
\definecolor{mypurple}{RGB}{111,0,255}
\definecolor{myred}{rgb}{1.0,0.0,0.0}
\definecolor{mygold}{rgb}{0.75,0.6,0.12}
\definecolor{myblue}{rgb}{0,0.2,0.8}
\definecolor{mydarkgray}{rgb}{0.,0.2,0.2}

\definecolor{lightred}{RGB}{255,235,235}
\definecolor{lightgreen}{RGB}{235,255,235}
\definecolor{lightblue}{RGB}{235,235,255}
\definecolor{lightcyan}{RGB}{235,255,255}
\definecolor{lightmagenta}{RGB}{255,235,255}
\definecolor{lightyellow}{RGB}{255,255,235}

\definecolor{qxkcolor}{RGB}{215,235,255}
\definecolor{softmaxcolor}{RGB}{230,235,255}
\definecolor{probxvcolor}{RGB}{255,255,235}

\definecolor{topkcolor}{RGB}{255,235,235}
\definecolor{zecolor}{RGB}{255,255,235}
\definecolor{dynacolor}{RGB}{235,255,255}

\definecolor{reviewcolor}{RGB}{0,0,200}

\newcommand{\ceil}[1]{\lceil #1 \rceil}

\newcommand{\pfrac}[2]{\frac{\partial #1}{\partial #2}}

\newcommand{\calL}{\mathcal{L}}
\newcommand{\calP}{\mathcal{P}}
\newcommand{\calS}{\mathcal{S}}

\DeclareMathOperator*{\argmin}{argmin}

\theoremstyle{plain}

\newtheorem{myclaim}{\textbf{Claim}}

\theoremstyle{definition}

\algdef{SE}[DOWHILE]{Do}{doWhile}{\algorithmicdo}[1]{\algorithmicwhile\ #1}%

\graphicspath{{./figs/}}

\newcommand{\name}{\texttt{L}$^\texttt{2}$\texttt{ight}\xspace}

\begin{document}

\pagestyle{plain} %

\title{
L$^\text{2}$ight: Enabling On-Chip Learning\\for Optical Neural Networks via Efficient\\ \textit{in-situ} Subspace Optimization
}

\author
{
Jiaqi Gu,
Hanqing Zhu,
Chenghao Feng,
Zixuan Jiang,
Ray T. Chen,
David Z. Pan
\\
ECE Department, University of Texas at Austin\\
\small\textit{\{jqgu, hqzhu,fengchenghao1996,zixuan\}@utexas.edu},
\small\textit{\{chen, dpan\}@ece.utexas.edu}
}

\maketitle
\begin{abstract}
\label{abstract}
Silicon-photonics-based optical neural network (ONN) is a promising hardware platform that could represent a paradigm shift in efficient AI with its CMOS-compatibility, flexibility, ultra-low execution latency, and high energy efficiency.
\emph{In-situ} training on the online programmable photonic chips is appealing but still encounters challenging issues in on-chip implementability, scalability, and efficiency.
In this work, we propose a closed-loop ONN on-chip learning framework \name to enable scalable ONN mapping and efficient \textit{in-situ} learning.
\name adopts a three-stage learning flow that first calibrates the complicated photonic circuit states under challenging physical constraints, then performs photonic core mapping via combined analytical solving and zeroth-order optimization.
A subspace learning procedure with multi-level sparsity is integrated into \name to enable \textit{in-situ} gradient evaluation and fast adaptation, unleashing the power of optics for real on-chip intelligence.
Extensive experiments demonstrate our proposed \name outperforms prior ONN training protocols with \textbf{3-order-of-magnitude} higher scalability and over \textbf{30}$\times$ better efficiency, when benchmarked on various models and learning tasks.
This synergistic framework is the \emph{first} scalable on-chip learning solution that pushes this emerging field from \emph{intractable} to \emph{scalable} and further to \emph{efficient} for next-generation self-learnable photonic neural chips.
From a co-design perspective, \name also provides essential insights for hardware-restricted unitary subspace optimization and efficient sparse training.
We open-source our framework at \href{https://github.com/JeremieMelo/L2ight}{link}.
\end{abstract}

\section{Introduction}
\label{sec:Introduction}
The escalating scales of deep learning models and datasets have brought increased demand for computing capacities in electronic processors.
Stringent performance and efficiency constraints in practical applications raise a surging need to develop more efficient computing solutions.
As a promising substitute for conventional electronics, optical neural networks (ONNs) have attracted extensive research interests owing to their sub-nanosecond latency and attojoule/multiply-accumulate operation (MAC) energy efficiency~\cite{NP_NATURE2017_Shen, NP_PIEEE2020_Cheng,NP_Nature2020_Wetzstein, NP_NaturePhotonics2021_Shastri,NP_Nature2021_Xu,NP_Nature2021_Feldmann}, shown in Figure~\ref{fig:ONNEfficiency}.

However, robustness and trainability are still critical issues for photonic AI engines~\cite{NP_ICCAD2019_Zhao, NP_DATE2020_Gu, NP_ICCAD2020_Zhu}.
Due to the analog computing nature of ONNs, the photonic DNN model inevitably suffer from performance degradation or even complete malfunction~\cite{NP_ICCAD2019_Zhao,NP_ICCAD2020_Zhu} with the existence of manufacturing errors, non-ideal device controls, and undesired circuit noises, shown in Figure~\ref{fig:NoiseSensitivity}.
Though non-ideal effects can be simulated and considered during software training~\cite{NP_ICCAD2019_Zhao,NP_DATE2020_Gu} to improve noise tolerance, the variation simulation is physically inaccurate (especially with unknown process variations) and prohibitively expensive, shown in Figure~\ref{fig:SimulationRuntime}.

\begin{wrapfigure}[13]{r}{0.72\textwidth}
\begin{minipage}{0.72\textwidth}
\vspace{-10pt}
    \centering
     \subfloat[]{
    \includegraphics[width=0.325\textwidth]{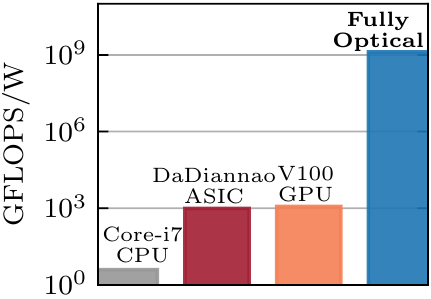}
    \label{fig:ONNEfficiency}
    }
    \hspace{-9pt}
    \subfloat[]{
    \includegraphics[width=0.315\textwidth]{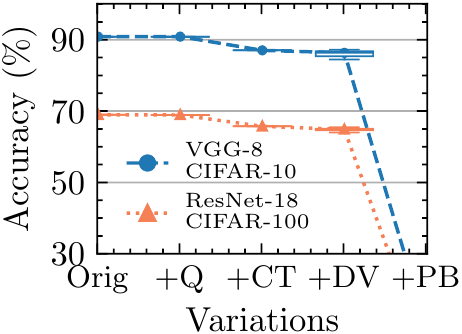}
    \label{fig:NoiseSensitivity}
    }
    \hspace{-8pt}
    \subfloat[]{
    \includegraphics[width=0.333\textwidth]{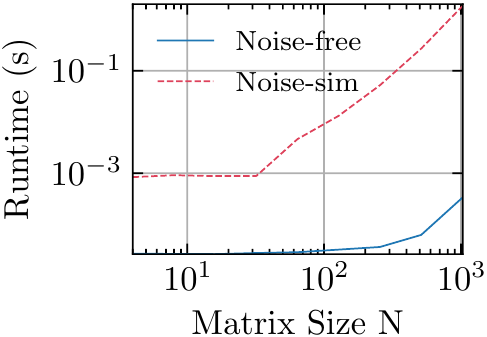}
    \label{fig:SimulationRuntime}
    }
    \caption{\small Comprehensive motivations. (a) Computational efficiency superiority of ONNs~\cite{NP_NATURE2017_Shen}. 
    (b) Noise sensitivity of ONNs (Q: 8-bit quantization, CT: crosstalk, DV: device variation, PB: phase bias).
    (c) Runtime of noise-free matrix multiplication vs. w/ noise simulation (Q+CT+DV).}
    \label{fig:Motivation}
\end{minipage}
\end{wrapfigure} 
Recently, on-device training has become an appealing trend towards adaptable and self-learning ONNs.
However, training on photonic neural chips is non-trivial and much less explored than on conventional platforms.
Prior work~\cite{NP_OE2019_Zhang, NP_OPTICA2018_Hughes, NP_DAC2020_Gu, NP_AAAI2021_Gu} only demonstrated small prototypes, and their scalability and efficiency are rather limited.
To push the limits of DNNs in optics, we propose an efficient three-stage learning framework \name that consists of variation-agnostic identity calibration, alternate projection-based parallel mapping, and multi-level sparse subspace learning.
The main contributions of this work are four-fold,
\begin{itemize}[leftmargin=*]
\setlength{\itemindent}{0.5em}
    \item \textbf{Scalability}. \emph{For the first time}, an ONN learning protocol can scale up to million-level parameters under practical circuit non-ideality, over 3-order-of-magnitude more scalable than prior arts.
    \item \textbf{Efficiency}. We explore multi-level sparsity in \emph{in-situ} gradient evaluation to trim down unnecessary on-chip training energy and runtime cost.
    \item \textbf{Learnability}. By trading redundant representability, our restricted subspace optimization can provide ONNs with enough adaptability for on-device self-learning and task transfer.
    \item \textbf{Robustness}. Various practical device noises and process variations are considered \emph{in situ} to facilitate noise-resilient photonic AI engines.
    \item To our best knowledge, this is the first framework that supports on-chip training on million-parameter ONNs, over \textbf{1000}$\times$ more scalable and \textbf{30}$\times$ more efficient than prior art.
    Our PyTorch~\cite{NN_pytorch2019} library for ONN on-chip training is open-sourced at \href{https://github.com/JeremieMelo/L2ight}{link}.
\end{itemize}

\section{Related Work}
\label{sec:Background}

\textbf{Optical Neural Network and Training Methods.}
One of recent ONN architectures adopts singular value decomposition (SVD) to implement matrix multiplication~\cite{NP_NATURE2017_Shen}, i.e., $y=\bm{W}x=\bm{U\Sigma V^*}x$.
Cascaded 2-by-2 optical devices, i.e., Mach-Zehnder interferometers (MZIs), are used to construct unitary matrices as the product of a series of 2-dimensional unitary rotators, $U(n) = D \prod^{2}_{i=n}\prod^{i-1}_{j=1} R_{ij}(\phi_{ij})$.
A detailed introduction to ONNs can be found in Appendix~\ref{sec:AppendixONNPrinciple}.
Beyond offline training~\cite{NP_DATE2020_Gu}, ONN on-chip training methods are proposed to offload the process back onto photonics~\cite{NP_OPTICA2018_Hughes,NP_DAC2020_Gu, NP_AAAI2021_Gu}, shown in Table~\ref{tab:CompareOnchipTrain}.
Brute-force device tuning (\texttt{BFT})~\cite{NP_NATURE2017_Shen, NP_JSTQE2020_Zhou} and evolutionary algorithms~\cite{NP_OE2019_Zhang} are applied to search MZI settings.
An adjoint variable method (\texttt{AVM})~\cite{NP_OPTICA2018_Hughes} is proposed to directly evaluate gradients using \emph{in-situ} light field monitoring.
Stochastic zeroth-order optimization (ZOO)~\cite{NP_DAC2020_Gu, NP_AAAI2021_Gu} is later applied to improve the training efficiency.
However, prior methods are hard to scale to larger ONNs either due to algorithmic inefficiency or unrealistic hardware complexity.
\begin{table}[h]
\centering
\caption{Scalability comparison with prior ONN on-chip training protocols in terms of \#Params they can handle, used algorithm, resolution requirement (\emph{Req.}), and circuit observability requirement.
\emph{Coh. I/O} is short for coherent input/output~\cite{NP_Optica2020_Miller,NP_NatureComm2021_Zhang}.
ZO, FO mean zeroth- and first-order methods.}
\label{tab:CompareOnchipTrain}
\resizebox{\textwidth}{!}{%
\begin{tabular}{l|cccccc}
\toprule
                   & BFT~\cite{NP_NATURE2017_Shen}       & PSO~\cite{NP_OE2019_Zhang}       & AVM~\cite{NP_OPTICA2018_Hughes}       & FLOPS~\cite{NP_DAC2020_Gu}      & MixedTrn~\cite{NN_AAAI2020_Bibi}   & \name     \\\midrule
\#Params             & $\sim$100 & $\sim$100 & $\sim$100 & $\sim$1000 & $\sim$2500 & $\sim$\textbf{10 M} \\
Algorithm              & ZO           & ZO          & FO          & ZO           & ZO           &   ZO+FO        \\
Resolution Req.    &    Medium       & High          & Medium          & High           & Med           &  Medium          \\
Observability Req. & Coh. I/O           &  Coh. I/O         &  Coh. I/O + Per device monitor         & Coh. I/O           & Coh. I/O           & Coh. I/O  \\\bottomrule       
\end{tabular}%
}
\end{table}

\textbf{Efficient NN Training Methods.}
Extensive work has been devoted to accelerating DNN training, among which an important branch is sparse backpropagation.
Previous methods mainly focus on approximating matrix multiplication by sparsifying the pre-activation gradients~\cite{NN_ICML2017_Sun}, forward and feedback matrices~\cite{NN_2018Arxiv_Menachem, NN_NeurIPS2020_Aamir}, and input feature maps~\cite{NN_ICLR2021_Oktay}.
Quantization to the pre-activation gradients is adopted in \cite{NN_2020CVPRW_wiedemann} to induce sparsity by trading off quantization steps and performance.
Other methods also exist, e.g., distributed and low-precision training~\cite{NN_NeurIPS2018_Banner, NN_ICML2018_Bernstein, NN_Arxiv2018_Jia}. 
However, they are not readily applicable to analog photonic engines, thus not in the scope of our discussion.

\textbf{Subspace Neural Networks.}
Subspace neural networks are special DNN models with restricted parameter space but demonstrate comparable representability to classical NNs.
Sparse NNs~\cite{NN_ICLR2016_Han,NN_NIPS2016_Wen}, low-rank NNs~\cite{NN_NIPS2014_Denton,NN_ICLR2015_Lebedev,NN_ICLR2016_Tai}, structured NNs~\cite{NN_MICRO2017_Ding,NN_ICLR2018_Li,NN_DATE2020_Wang}, Fourier-domain NNs~\cite{NP_ASPDAC2020_Gu,NP_CLEO2020_Miscuglio,NP_Optica2020_Miscuglio}, and general frequency-domain NNs~\cite{NP_TCAD2020_Gu} were introduced to trim down the redundancy in DNNs by restricting the NN structure, matrix rank, numerical resolution, etc.
In this work, we deeply explore the trade-off between ONN learnability, trainability, and efficiency in the restricted unitary subspace.

\textbf{Challenges of ONN On-Chip Training.}
As a unique hardware-restricted optimization problem, ONN \textit{in-situ} learning encounters fundamental challenges causing scalability issues in prior methods:
\begin{itemize}[leftmargin=*]
\setlength{\itemindent}{0.5em}
    \item \textbf{Lack of full-observability for \textit{in-situ} light field.}
    Tracking physical optical field on every waveguide in $\bm{U}$ and $\bm{V}^{\ast}$ is not scalable or practical when ONNs scale up.
    Per device light field monitoring and calibration~\cite{NP_Optica2014_Grillanda,NP_OPTICA2018_Hughes} involves intractable hardware complexity.
    In practice, only $\bm{\Sigma}$ can be precisely monitored and efficiently tuned.
    \item \textbf{Limited input/output observability}.
    In photonic tensor cores, for efficiency consideration, only the final output signals after $\bm{U\Sigma V}^{\ast}$ can be coherently detected. 
    Intermediate signals of a single unitary projection can not be easily read out without extra hardware support.
    \item \textbf{Inaccessible gradients for most control variables.}
    Due to the above two limitations, it is challenging to obtain true derivatives w.r.t. the MZI rotation phases in $\bm{U}$ and $\bm{V}^{\ast}$~\cite{NP_NATURE2017_Shen,NP_DAC2020_Gu,NP_AAAI2021_Gu}, casting fundamental \textit{in-situ} optimization difficulty as ONN scales up.
\end{itemize}

To enable \emph{in-situ} self-learning for ONNs, the proposed synergistic framework \name provides a scalable, efficient, and on-chip-implementable solution that overcomes those hardware restrictions.

\section{Synergistic ONN On-Chip Learning Framework \name}
\label{sec:Method}

In this section, we give a formal description of the ONN on-chip training problem and detailed demonstration of our proposed three-stage learning flow \name, shown in Figure.~\ref{fig:ONNTrainFlow}.
\begin{figure}[H]
    \centering
    \includegraphics[width=0.97\textwidth]{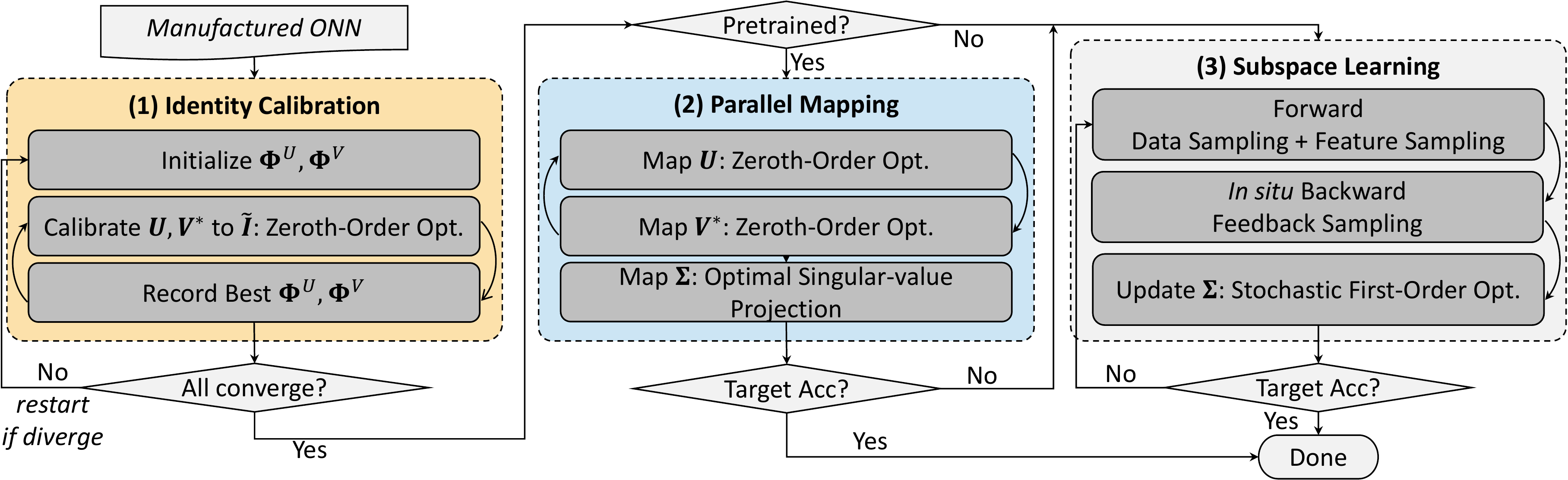}
    \caption{~Proposed three-stage ONN on-chip learning flow \name.}
    \label{fig:ONNTrainFlow}
    \vspace{-10pt}
\end{figure}

\subsection{Understanding the ONN On-Chip Learning Problem}
\label{sec:ProblemFormulation}
The ONN that supports on-chip learning is shown in Figure~\ref{fig:ONNTrainArch}, constructed by local storage, control units, interconnects, and photonic tensor cores with coherent I/O~\cite{NP_Optica2020_Miller,NP_NatureComm2021_Zhang} and wavelength-division multiplexing (WDM)~\cite{NP_CommuMag2010_Yu, NP_OE2014_Tan} for parallel processing.
\begin{figure*}
    \centering
    \includegraphics[width=0.98\textwidth]{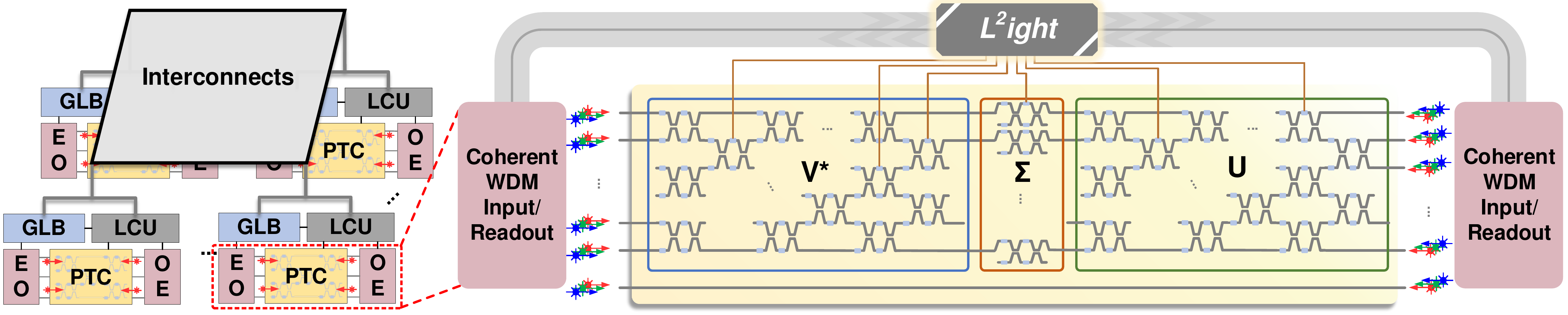}
    \caption{~\small ONN architecture. PTC: photonic tensor core, GLB: global buffer, LCU: local control unit, EO: electrical-to-optical conversion.}
    \label{fig:ONNTrainArch}
    \vspace{-10pt}
\end{figure*}
The target is to optimize MZI phases $\bm{\Phi}$ directly on chip under variations.
Formally the \emph{hardware-restricted} learning problem is,
\begin{equation}
    \small
    \label{eq:Objective}
    \begin{aligned}
    &\bm{\Phi}^*=\argmin_{\bm{\Phi}}~\mathcal{L}\big(\bm{W}(\bm{\Omega}\bm{\Gamma}\mathcal{Q}(\bm{\Phi})+\bm{\Phi}_b);\mathcal{D}_{trn}\big),\\
    \text{s.t.}~~&\bm{W}(\bm{\Phi})=\big\{\bm{W}_{pq}(\bm{\Phi}_{pq})\big\}_{p=0,q=0}^{p=P-1,q=Q-1},\quad \bm{W}_{pq}(\bm{\Phi}_{pq})=\bm{U}_{pq}(\bm{\Phi}_{pq}^U)\bm{\Sigma}_{pq}(\bm{\Phi}_{pq}^S)\bm{V}_{pq}^*(\bm{\Phi}_{pq}^V),\\
    &\bm{U}_{pq}(\bm{\Phi}_{pq}^U)=\bm{D}_{pq}^U\prod_{i=k}^2\prod_{j=1}^{i-1}\bm{R}_{pqij}(\phi_{pqij}^U),\quad \bm{V}_{pq}^*(\bm{\Phi}_{pq}^V)=\bm{D}_{pq}^V\prod_{i=k}^2\prod_{j=1}^{i-1}\bm{R}_{pqij}(\phi_{pqij}^V),\\
    &\bm{\Sigma}_{pq}(\bm{\Phi}_{pq}^S)=\max(|\bm{\Sigma}_{pq}|)\texttt{diag}(\cdots,\cos{\phi_{pq,i}^S},\cdots), \quad \bm{\Phi}_b\sim\mathcal{U}(0,2\pi), ~\bm{\Gamma}\sim\mathcal{N}(\gamma, \sigma_{\gamma}^2).
    \end{aligned}
\end{equation}
The linear projection in an ONN adopts blocking matrix multiplication, where the $M\times N$ weight matrix is partitioned into $P \times Q$ blocks of size $k \times k$.
During the optimization of $\bm{\Phi}$, we jointly consider control resolution limits $\mathcal{Q}(\cdot)$~\cite{NP_DATE2020_Gu, NP_NATURE2017_Shen}, device process variations $\bm{\Gamma}$~\cite{NP_DATE2020_Gu, NP_DAC2020_Gu, NP_AAAI2021_Gu}, thermal crosstalk among adjacent devices $\bm{\Omega}$~\cite{NP_DAC2020_Gu,NP_ICCAD2020_Zhu}, and unknown phase bias due to manufacturing error $\bm{\Phi}_b$ for \emph{in-situ} robustness-aware training.
A detailed non-ideality analysis is in Appendix~\ref{sec:AppendixNonideality}.
For practicality, robustness, and convergence consideration, we select $k$=9, which is explained in Appendix~\ref{sec:AppendixScaling}.
\subsection{Identity Calibration (IC): Variation-Agnostic Circuit State Preparation}
\label{sec:IdentityCalibration}
After manufacturing, unknown process variations in waveguides make the initial state of PTCs unpredictable~\cite{NP_OFC2018_Timurdogan, NP_ICCAD2020_Zhu}.
A primary task is to prepare $\bm{U}$ and $\bm{V}^{\ast}$ to be identity matrices.
However, the calibration problem, i.e., $\min_{\bm{\Phi^U},\bm{\Phi^V}}~\sum_{p,q}\big(\|\bm{U}_{pq}(\bm{\Phi}^U_{pq})-\bm{I}\|_2^2+\|\bm{V}^{\ast}_{pq}(\bm{\Phi}^V_{pq})-\bm{I}\|_2^2\big)$, is not solvable given the observability and controllability constraints on $\bm{U}$ and $\bm{V}^{\ast}$.
The closest auxiliary problem that we can solve is the one with absolute operations on unitaries, i.e., $\min_{\bm{\Phi^U},\bm{\Phi^V}}~\sum_{p,q}\big(\||\bm{U}_{pq}(\bm{\Phi}^U_{pq})|-\bm{I}\|_2^2+\||\bm{V}^{\ast}_{pq}(\bm{\Phi}^V_{pq})|-\bm{I}\|_2^2\big)$.
We denote those two mean square errors as $MSE^U$ and $MSE^V$.
We rewrite it as a surrogate minimization of $\calL_{IC}$ that can lead to the same solution,
\begin{wrapfigure}[12]{rH}{0.52\textwidth}
\begin{minipage}{0.52\textwidth}
\vspace{-20pt}
\centering
    \subfloat[]{
    \includegraphics[width=0.46\textwidth]{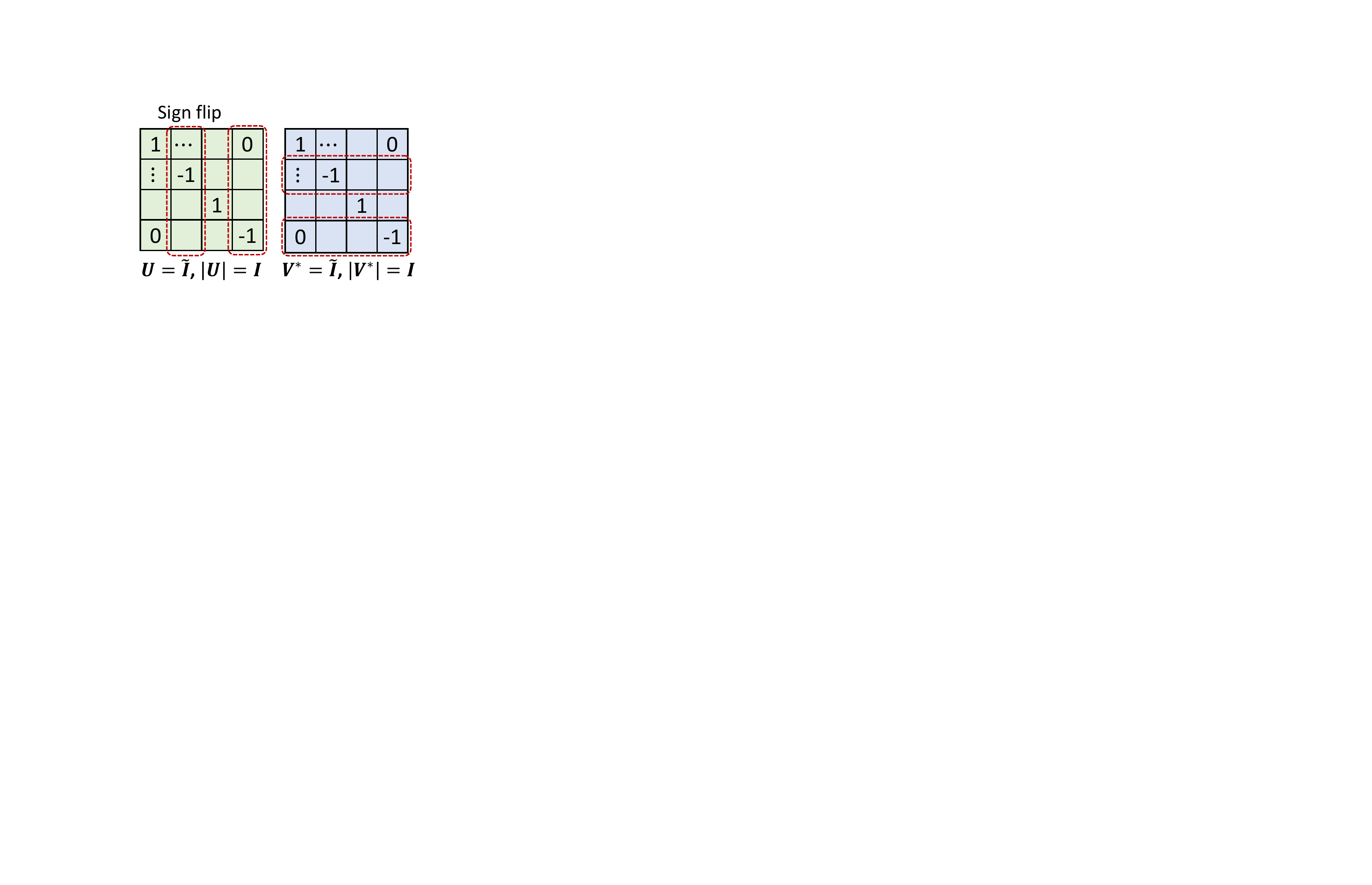}
    \label{fig:IdentityCalibrationSignFlip}}
    \subfloat[]{
    \includegraphics[width=0.48\textwidth]{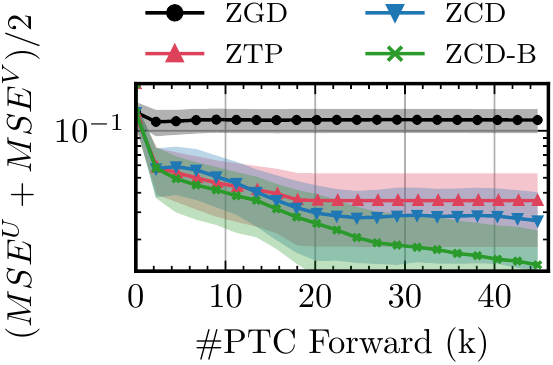}
    \label{fig:IdentityCalibrationCurve}}
    \label{fig:IdentityCalibration}
    \caption{\small 
    (a) Identity calibration with sign flip. 
    (b) Different ZO optimizers on identity calibration. (ZGD: ZO gradient descent with momentum, ZCD: ZO coordinate descent, ZTP: ZO three-point.
    \emph{B} is best solution recording.)
    }
\end{minipage}
\end{wrapfigure}
\begin{equation}
    \small
    \label{eq:IdentityCalibration}
    \min_{\bm{\Phi}}\!\sum_{p,q}\|\bm{U}_{pq}(\bm{\Phi}^U_{pq})\bm{\Sigma}_{pq}\bm{V}^{\ast}_{pq}(\bm{\Phi}^V_{pq})\bm{\Sigma}_{pq}^{-1}\!-\!I\|.
\end{equation}

The optimal solution for this auxiliary problem is $\bm{U}=\bm{V}^{\ast}=\tilde{\bm{I}}$, where $\tilde{\bm{I}}$ is not guaranteed to be an identity matrix but a more general \emph{sign-flipping matrix} with \emph{arbitrary and unobservable sign flips} on the same columns in $\bm{U}$ and rows in $\bm{V}^{\ast}$, shown in Figure~\ref{fig:IdentityCalibrationSignFlip}.
We adopt zeroth-order optimization (ZOO) on $\bm{\Phi}^{U}$ and $\bm{\Phi}^{V}$ to calibrate $\bm{U}$ and $\bm{V}^{\ast}$ to approach $\tilde{\bm{I}}$, shown in Figure~\ref{fig:IdentityCalibrationCurve}.
We show the converged solution of Eq.~\eqref{eq:IdentityCalibration} with unobservable sign flips and suboptimality only has marginal impacts on the following training procedure in later sections.

\subsection{Parallel Mapping (PM): Alternate Projection-based Model Deployment}
\label{sec:ParallelMapping}
The target is to map the pre-trained weights $\bm{W}$ onto photonic MZI meshes $\widetilde{\bm{W}}(\bm{\Phi})$ with high %
\begin{wrapfigure}[12]{r}{0.6\textwidth}
\begin{minipage}{0.6\textwidth}
\vspace{-20pt}
\centering
    \subfloat[]{\includegraphics[width=0.49\textwidth]{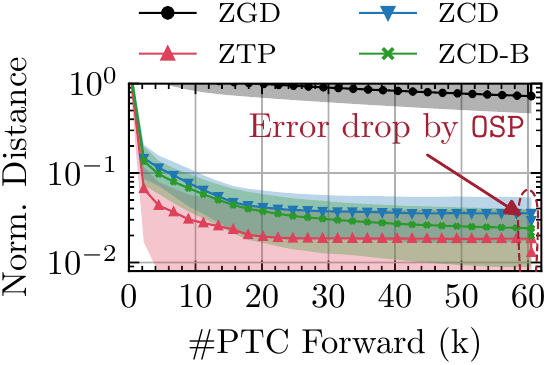}
    \label{fig:ParallelMappingLossCompare}}
    \subfloat[]{\includegraphics[width=0.49\textwidth]{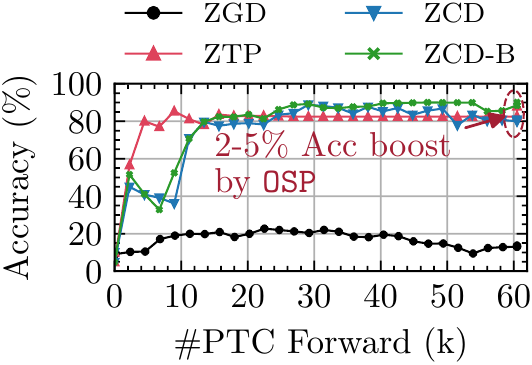}
    \label{fig:ParallelMappingAccCompare}}
    \caption{
    \small ZTP and ZCD-B perform the best in parallel mapping.
    The optimal singular-value projection leads to significant error drop and accuracy jump.)
    \label{fig:ParallelMappingCompare}
    }
\end{minipage}
\end{wrapfigure}
fidelity.
We formulate the parallel mapping as a batched $k\times k$-block-wise regression problem,
\begin{equation}
    \small
    \label{eq:Regression}
    \begin{aligned}
     \min_{\bm{\Phi}}~\sum_{p,q}\|\widetilde{\bm{W}}_{pq}(\bm{\Phi}_{pq})-\bm{W}_{pq}\|_2^2.
    \end{aligned}
\end{equation}

As analyzed before, $\frac{\partial\bm{W}}{\partial\bm{\Phi}^U}$ and $\frac{\partial\bm{W}}{\partial\bm{\Phi}^V}$ are too expensive to compute \textit{in situ}.
We propose a parallel mapping flow with alternate zeroth-order optimization on $\bm{\Phi}^U$ and $\bm{\Phi}^V$.
After convergence, we will perform analytical optimal singular-value projection (\texttt{OSP}) to minimize the regression error given fixed singular vectors.

We show why \texttt{OSP} gives the optimal solution under sign flips and how to perform it on the PTC.
\begin{myclaim}
Optimal singular-value projection (\texttt{OSP}): the optimal singular value problem, i.e., $\bm{\Sigma}_{opt}=\argmin_{\bm{\Sigma}}~\|\bm{U\Sigma V}^{\ast}-\bm{W}\|$, can be analytically solved on-chip with arbitrary and unknown sign flip.
\end{myclaim}
\begin{proof}
\vspace{-10pt}
\begin{equation}
    \small
    \label{eq:OptimalProjection}
    \begin{aligned}
      \bm{\Sigma}_{opt}=&\texttt{diag}\big(\bm{U}^{-1}\bm{W}(\bm{V}^{\ast})^{-1}\big)=\texttt{diag}\big(\bm{U}^{\ast}\bm{W}\bm{V}\big)
      =\texttt{diag}\big((\tilde{\bm{I}}^{\ast}\bm{V}^{\ast}\bm{W}^{\ast}\bm{U}\tilde{\bm{I}})^{\ast}\big).
    \end{aligned}
\end{equation}
\texttt{OSP} can be directly achieved using the limited operation set, i.e., $\{\bm{U},\bm{U}^{\ast},\bm{V},\bm{V}^{\ast}\}$, supported by the reciprocal PTC itself.
Specifically, we configure $\bm{V}^{\ast}=\tilde{\bm{I}}$ and $\bm{\Sigma}=\bm{I}$, and shine in a coherent WDM light beam that carries $\bm{W}$ from right ports.
Since the coherent photonic circuit is reciprocal~\cite{NP_Optica2020_Miller}, we can read $\tilde{\bm{I}}\bm{U}^{\ast}\bm{W}$ on the left ports.
Then we configure $\bm{U}=\tilde{\bm{I}}$ and $\bm{\Sigma}=\bm{I}$, and shine in its adjoint field from left, i.e., $\bm{W}^{\ast}\bm{U}\tilde{\bm{I}^{\ast}}$.
We can directly read out the projected optimal diagonal on the right because the sign flips in the unitary matrices naturally cancel out on the diagonal.
\end{proof}

Figure~\ref{fig:ParallelMappingCompare} compares different ZO optimizers on this task.
Coordinate-wise optimizers (\texttt{ZCD}~\cite{NN_NIPS2016_Lian} and \texttt{ZTP}~\cite{NN_SIAM2020_Bergou}) outperform the gradient-based \texttt{ZGD}~\cite{NN_SIAM2013_Ghadimi} with higher accuracy and convergence speed.
This procedure is highly parallel and efficient since the mapping involves \emph{no stochasticity} and only happens \emph{locally} within each PTC.
We can also observe that \texttt{OSP} effectively reduces the normalized matrix distance ($\|\bm{W}-\widetilde{\bm{W}}\|_2^2/\|\bm{W}\|_2^2$) and boosts the accuracy by 2-5\% almost for free.

\subsection{Subspace Learning: Hardware-Aware Multi-Level Sparse Training}
Besides mapping from an offline-trained model, \name also supports \emph{in-situ} self-learning fully on chip.
We name this feature as \emph{subspace learning}. 
To make \name hardware-aware, we trade expensive full-space trainability for efficient subspace gradient evaluation, i.e., $\pfrac{\calL}{\bm{\Sigma}}$
which coincides with the general frequency-domain ONNs~\cite{NP_ASPDAC2020_Gu,NP_TCAD2020_Gu} and subspace NN design concept~\cite{NN_KDD2018_Sun}.
Since this learning stage involves stochasticity, it turns out to be the efficiency bottleneck, especially the backward pass.
Hence, we explore multi-level sparsity for efficient \emph{in-situ} gradient approximation.

\subsubsection{\emph{In-situ} Subspace Gradient Acquisition via Reciprocity in Optics}
\begin{wrapfigure}[14]{r}{0.48\textwidth}
\vspace{-15pt}
    \centering
    \includegraphics[width=0.48\textwidth]{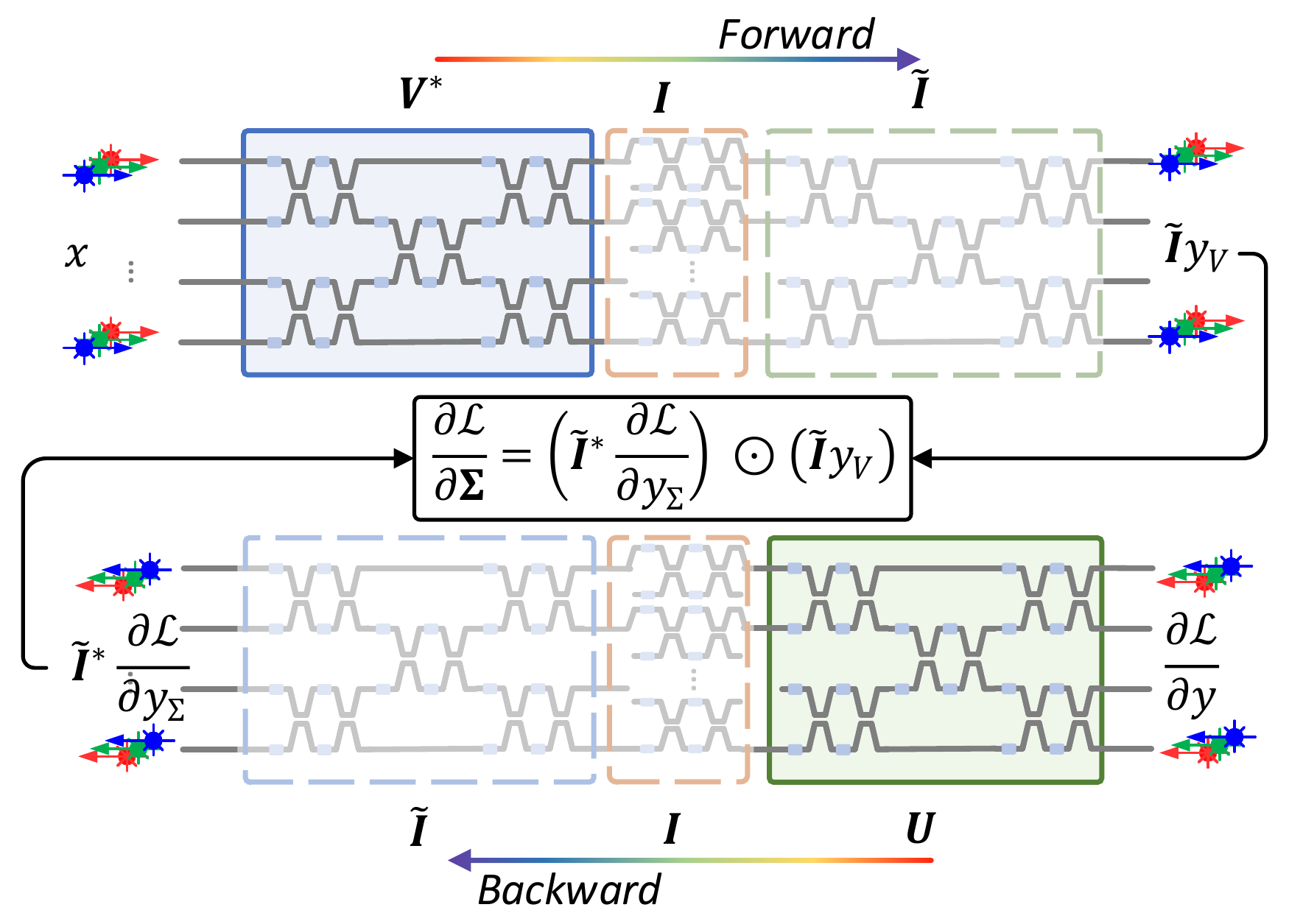}
    \caption{~\small \textit{In-situ} subspace gradient acquisition.}
    \label{fig:BackProp}
\end{wrapfigure}
The conventional way to compute first-order gradients w.r.t. $\bm{\Sigma}$ is $\frac{\partial\mathcal{L}}{\partial\bm{\Sigma}}=\texttt{diag}\big(\bm{U}^{\ast}\frac{\partial\mathcal{L}}{\partial\bm{W}}\bm{V}\big)$.
However, $\frac{\partial\mathcal{L}}{\partial\bm{W}}=\frac{\partial\mathcal{L}}{\partial\bm{y}}\bm{x}^T$ requires arbitrary matrix multiplication, which is not implementable by weight-stationary PTCs.
Hence, we remap it as,
\begin{equation}
    \small
    \label{eq:BackProp}
    \begin{aligned}
        \frac{\partial\mathcal{L}}{\partial\bm{\Sigma}}=\frac{\partial\mathcal{L}}{\partial\bm{y}_{\Sigma}}\odot\bm{y}_{V}=\big(\tilde{\bm{I}}\bm{U}^{\ast}\frac{\partial\mathcal{L}}{\partial\bm{y}}\big)\odot(\tilde{\bm{I}}\bm{V}^{\ast}\bm{x}).
    \end{aligned}
\end{equation}
By shining in coherent WDM beams carrying the inputs and upstream gradients forward and backward through the reciprocal PTCs, respectively, as shown in Figure~\ref{fig:BackProp}, the weight gradients can be efficiently obtained with lightweight element-wise multiplication $\odot$, which can be offloaded to electrical control units.
Note that $\tilde{\bm{I}}$ naturally cancels out by the Hadamard product with no impacts on gradient fidelity.

\subsubsection{Multi-Level Sparse Subspace Learning}
\label{sec:MultiLevelSparsity}
Inspired by sparse backpropagation methods~\cite{NN_ICML2017_Sun, NN_ICLR2021_Oktay, NN_NeurIPS2019_Wang, NN_NIPS2016_Arild, NN_NeurIPS2020_Aamir}, we propose multi-level sparse subspace learning to cut down both energy cost and total time steps in on-chip gradient evaluation.

\textbf{Balanced Feedback Sampling.}
To improve the efficiency of the error feedback process, i.e., $\bm{W}^T\pfrac{\calL}{y}$,
as shown in Figure~\ref{fig:BalancedFeedbackSampling}, we sample the feedback matrix $\bm{W}^T\in\mathbb{R}^{N\times M}$ with a structured sparse mask $\calP_W=c_W(\calS_W\otimes\bm{1})$ generated by the Kronecker product between a boolean mask $\calS_W\in\{0,1\}^{Q\times P}$ with sparsity $\alpha_W$ and an all-ones matrix $\bm{1}$, where the scaling factor $c_W$ is set to $\frac{1}{\alpha_W}=\frac{PQ}{\texttt{Tr}(\calS_W^T\calS_W)}$ for unbiased estimation, proven in Appendix~\ref{sec:AppendixUnbiasedGradient}.
The efficiency benefits come from two aspects:
(1) the structurally masked PTCs are entirely idle, directly saving energy,
and (2) the product accumulation depth/step is reduced by a factor of $\alpha_W$, effectively trimming time steps.

However, two major downsides exist on traditional \texttt{uniform} and layer-wise \texttt{topk} sampling~\cite{NN_NeurIPS2020_Aamir}.
\begin{wrapfigure}[16]{rH}{0.37\textwidth}
\begin{minipage}{0.365\textwidth}
\vspace{-6pt}
\centering
    \includegraphics[width=0.98\textwidth]{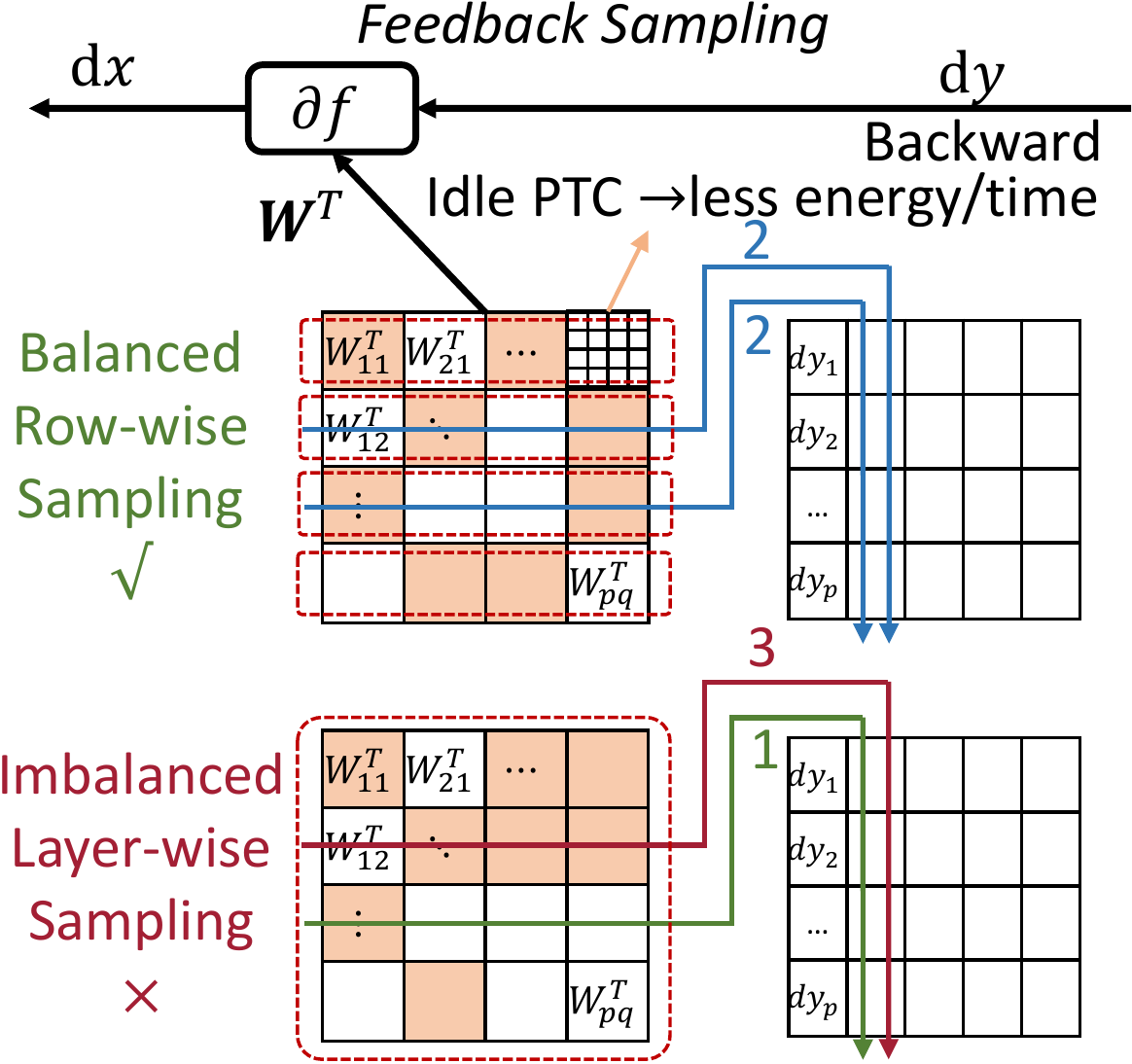}
    \caption{\small Balanced v.s. imbalanced feedback matrix sampling.
    }
    \label{fig:BalancedFeedbackSampling}
\end{minipage}
\end{wrapfigure}
First, on a backward path, multiple feedback sampling operators will be cascaded,
such that importance-unaware \texttt{uniform} sampling can lead to an exponentially large variance~\cite{NN_ICLR2021_Oktay}.
Second, \texttt{topk} sampling is overly greedy and tends to break the load balance as the feedback latency can be bottlenecked by the longest partial product accumulation path, shown in Figure~\ref{fig:BalancedFeedbackSampling}.
To tackle this, we propose a balanced top-K sampling (\texttt{btopk}) to draw $\mathcal{S}_W$ from a guided distribution that locally prefers blocks with large Frobenius norm, which can be efficiently evaluated by $\|\bm{W}_{pq}\|_{\mathcal{F}}^2=\texttt{Tr}(|\bm{\Sigma}_{pq}|^2)$.
It strikes a \emph{balance between gradient variance and bias} by fine-grained row-wise top-K sampling and \emph{eliminates load-imbalance} by guaranteeing the same sparsity for different rows of $\bm{W}^T$, i.e., $\sum_{p}\calS_W(1,:)=\sum_{p}\calS_W(2,:)=\cdots=\sum_{p}\calS_W(Q,:)$.
Figure~\ref{fig:FeedbackGradientAngular},~\ref{fig:FeedbackGradientAngularLayer} shows the gradient approximation fidelity in terms of average angular similarity~\cite{NN_Arxiv2018_Cer} and normalized matrix distance.
Our \texttt{btopk}-sampled weight gradients align well with the true gradients.
With the unbiased (\texttt{exp}) normalization factor $\alpha_W$, \texttt{btopk} shows the best gradient angular similarity and inference accuracy compared with others.

\textbf{Information-Preserving Column Sampling.}
\begin{figure}
    \centering
    \subfloat[]{
    \includegraphics[width=0.215\textwidth]{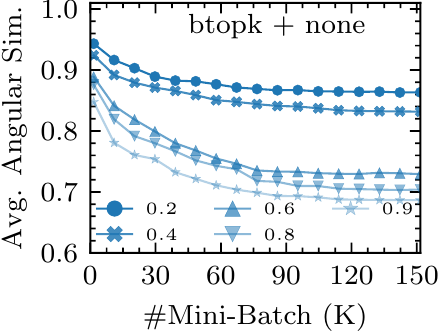}
    \label{fig:FeedbackGradientAngular}}
    \subfloat[]{
    \includegraphics[width=0.265\textwidth]{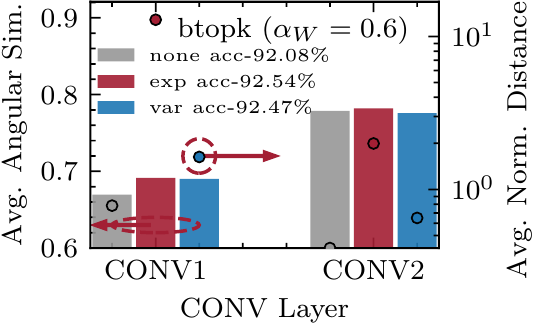}
    \label{fig:FeedbackGradientAngularLayer}}
    \subfloat[]{
    \includegraphics[width=0.215\textwidth]{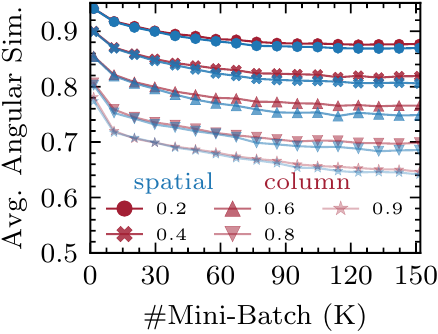}
    \label{fig:ConvFeatureGradientAngular}}
    \subfloat[]{
    \includegraphics[width=0.265\textwidth]{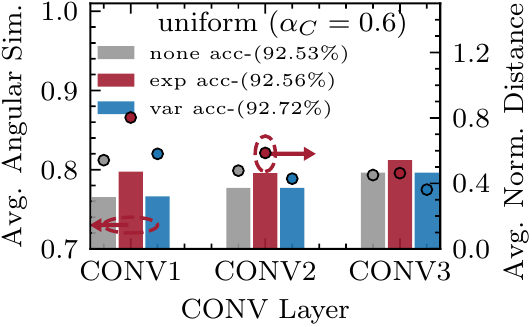}
    \label{fig:ConvFeatureGradientAngularLayer}}
    \caption{~\small
    Average gradient angular similarity with different feedback sparsity (a) and three normalization methods (b).
    \emph{none, exp,} and \emph{var} represents no, expectation-maintained, and variance-maintained normalization.
    Average gradient angular similarity with spatial and column sampling (c) and three normalization methods (d).}
    \vspace{-10pt}
\end{figure}
Input feature sparsification can also effectively cut down the gradient evaluation cost~\cite{NN_NeurIPS2020_Aamir,NN_ICLR2021_Oktay}, especially for costly CONV layers.
However, with traditional \emph{spatial sampling} (SS)~\cite{NN_NeurIPS2020_Aamir,NN_ICLR2021_Oktay}, the input feature map $x$ barely maintains its sparsity regularity after being transformed to flattened patches $X$ via \emph{im2col} if the kernel size is larger than 1, shown in Figure~\ref{fig:ConvFeatureSample}.
Hence, we propose a novel \emph{column sampling} (CS) as a better solution.
We sample $X$ using  %
\begin{wrapfigure}[18]{rH}{0.48\textwidth}
\begin{minipage}{0.48\textwidth}
\vspace{-10pt}
\centering
    \includegraphics[width=0.99\textwidth]{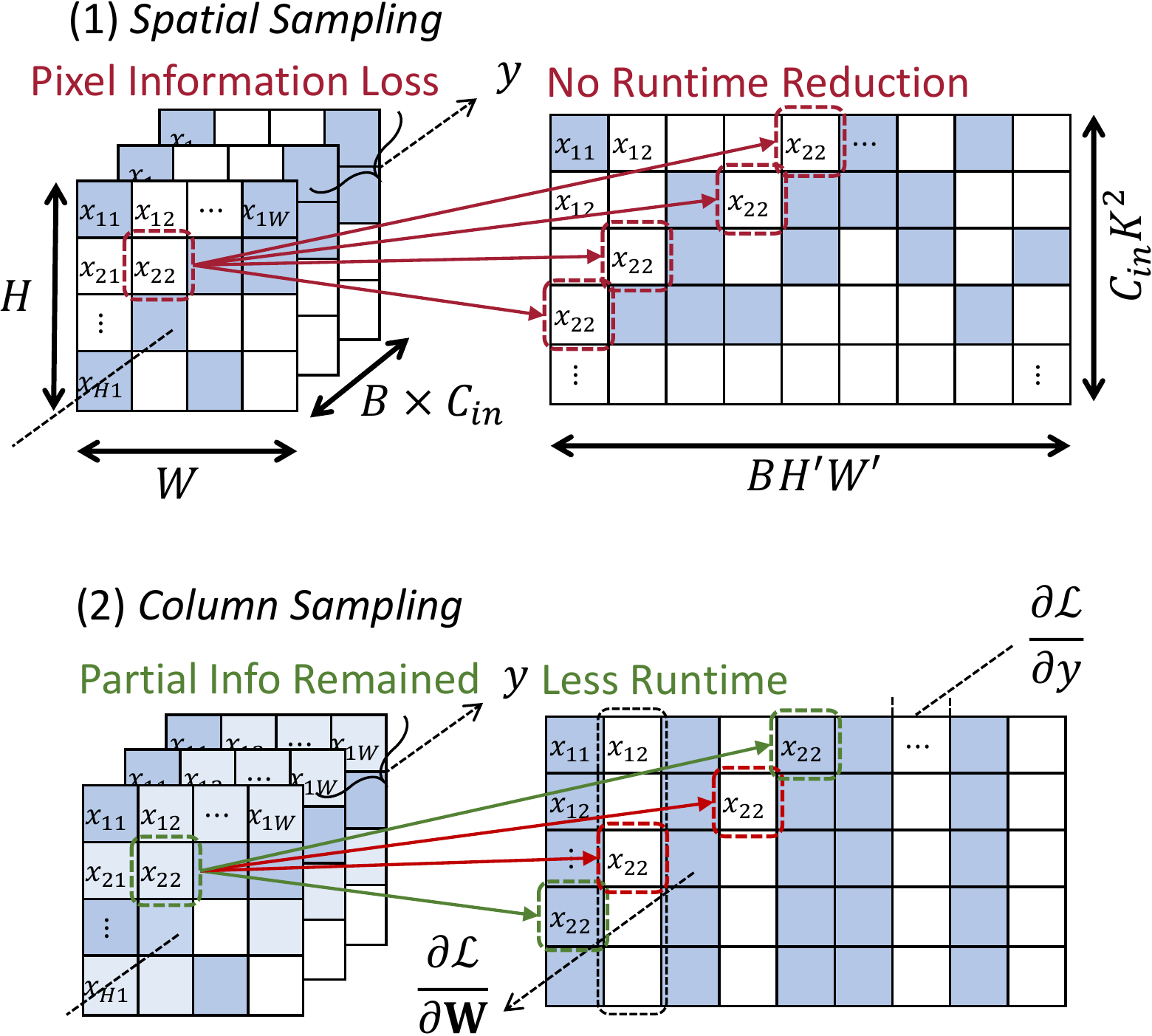}
    \caption{~\small Spatial and column sampling for CONV.}
    \label{fig:ConvFeatureSample}
\end{minipage}
\end{wrapfigure}
a mask $\calS_C\{0,1\}^{H'W'}$ with a uniform sparsity $\alpha_C$, which is shared across batches with negligible overhead.
This leads to both information preservation and efficiency improvement.
First, in Figure~\ref{fig:ConvFeatureSample}, a pixel appears in multiple columns, such that partial information can be maintained after column sampling.
Second, this highly-structured column dropping directly translates to less PTC forward energy and fewer partial gradient accumulation steps.
In contrast, with a spatial mask $\calS_S$ and spatial sparsity $\alpha_S$, the masked pixel will be completely dropped with poor regularity after \emph{im2col}, at the cost of large variance due to information loss and almost no runtime improvement on this dense linear projection engines.
Note that for CONV1$\times$1, CS turns out to be equivalent to SS, which can simultaneously save memory and runtime.
Figures~\ref{fig:ConvFeatureGradientAngular},~\ref{fig:ConvFeatureGradientAngularLayer} show that our proposed CS can obtain better gradient approximation fidelity than prior SS.
Different normalization has small effects on model accuracy since feature sampling only happens locally within each layer, without any variance cascade effect.
Note that simultaneous scaling by $\alpha_W$ and $\alpha_C$ tends to generate overly-confident gradient approximation, which empirically leads to harmful gradient variance.
Hence, we will adopt $\alpha_C$=1 in all experiments.

\textbf{Data Sampling.}
After parallel mapping, the ONN is initialized fairly close to the target pre-trained model. 
It is reasonable and intuitive to calibrate it with a representative calibration set instead of the entire training set.
Inspired by the mini-batch dropping (\texttt{SMD}) technique~\cite{NN_NeurIPS2019_Wang}, we integrate this \texttt{SMD} technique into our framework to further explore data-level sparsity.
Within one training epoch, we randomly skip each iteration with probability $\alpha_D$, directly translating to training time reduction.

\subsection{Complexity Analysis of Three Stages in \name}
We assume the total step in IC, PM, and SL is $T_1$, $T_2$, and $T_3$, respectively. 
The ONN has $L$ layers, each including an $N\times N$ weight matrix partitioned into multiple $k \times k$ blocks.

\noindent\textbf{Identity Calibration and Parallel Mapping.}~ 
Each block optimizes $k(k-1)$ phases using ZOO. 
All $LN^2/k^2$ blocks are optimized in parallel. 
The total step is $2k(k-1)T_1$ for IC and $2LN^2(k-1)T_2/k+3$ for PM. 
The total PTC call is around $2LN^2T_1$ or $2LN^2T_2$ for IC and PM, respectively.

\noindent\textbf{Subspace Learning.}~ We assume the feature map size is $H\times W$ with a batch size of $B$. 
The detailed complexity analysis is given in Appendix~\ref{sec:AppendixCost}. The total step is approximately $T_3LNBHW/k$.

According to our training cost profiler, IC and PM in total is 3-order-of-magnitude cheaper than the SL stage, since the batched parallel regression is deterministic and data-independent.

\vspace{-.05in}
\section{Results}
\label{sec:ExperimentalResults}
\subsection{Experiment Setup}
\label{sec:ExpSetup}
\textbf{Datasets.} We evaluate \name on Vowel~\cite{NN_Vowel1989}, MNIST~\cite{NN_MNIST1998}, FashionMNIST~\cite{NN_FashionMNIST2017}, CIFAR-10, CIFAR-100~\cite{NN_cifar2009}, and TinyImagenet~\cite{NN_imagenet2009}.
On CIFAR-10/100 and TinyImagenet, we adopt random crop, flip, color jittering for augmentation.

\textbf{Models.} We evaluate on a customized MLP (8-16-16-4)~\cite{NP_AAAI2021_Gu} on Vowel, CNN-S (CONV8K3S2-CONV6K3S2-FC10)~\cite{NP_AAAI2021_Gu} on MNIST, a CNN-L (\{CONV64K3\}$\times$3-Pool5-FC10) on FashionMNIST, and VGG-8~\cite{NN_NN2018_Dong}/ResNet-18~\footnote{\url{https://github.com/kuangliu/pytorch-cifar}}~\cite{NN_CVPR2016_He} on CIFAR-10/100.
CNN-L/FashionMNIST is used for ablation studies.
VGG-8/ResNet-18 on CIFAR-10/100 are used for accuracy and efficiency comparison.
Training details can be found in Appendix~\ref{sec:AppendixTrainDetails}.

\textbf{Efficiency Evaluation.}
We assume fully parallel $9\times9$-blocking matrix multiplication in photonic tensor cores and sequential partial product accumulation in electronics.
All experiments and performance measurements are based on software simulation with various noise modeling.
Our simulator counts the total number of PTC calls as the normalized energy indicator and the longest accumulation path as the normalized latency/runtime indicator.
Details of profiling can be found in Appendix~\ref{sec:AppendixCost}.

\subsection{Main Results}
\label{sec:ComparePriorWork}
\textbf{Scalability Comparison with Prior ONN Learning Protocols.}
Figure~\ref{fig:CompareScalability} compares \name with  %
\begin{wrapfigure}[10]{r}{0.55\textwidth}
\begin{minipage}{0.55\textwidth}
\vspace{-10pt}
    \centering
    \includegraphics[width=0.99\textwidth]{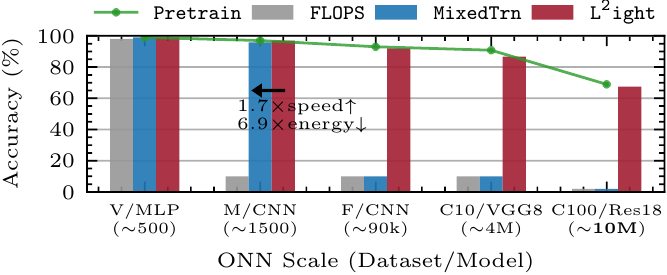}
    \vspace{-15pt}
    \caption{\small Compare scalability with prior protocols~\cite{NP_DAC2020_Gu, NP_AAAI2021_Gu}.}
    \label{fig:CompareScalability}
\end{minipage}
\end{wrapfigure}
two SOTA ONN on-chip training protocols, \texttt{FLOPS}~\cite{NP_DAC2020_Gu} and \texttt{MixedTrn}~\cite{NP_AAAI2021_Gu}.
For ZO methods, i.e., FLOPS and MixedTrn, we count the energy and latency of forward PTC query in Appendix~\ref{sec:AppendixCost}.
Prior protocols can only handle toy models ($\sim$2,000 params) given their algorithmic inefficiency and instability, while our \name shows $>$1,000$\times$ higher scalability to handle large ONNs ($\sim$10 M) on challenging tasks with comparable accuracy to full-space pre-trained models.
Though \texttt{MixedTrn} achieves comparable accuracy to \name on small benchmarks, we are still 1.7$\times$ faster and 6.9$\times$ more energy efficient.

The superiority of \name provides three important insights:
(1) \emph{decoupling ZOO from stochasticity} and \emph{partitioning a large-scale regression problem into a batch of sub-tasks} can greatly mitigate the curse of dimensionality both in convergence and efficiency.
(2) \emph{mapping before learning} can fully leverage the pre-trained model to reduce the learning cost.
Prior methods have to learn from a severely corrupted solution under variations, while \name recovers most accuracy via mapping, leaving a very light workload for subspace learning.
(3) Restricted subspace learning provides \emph{adequate degree of freedom} for training from scratch and task transfer.
Also, its \emph{compatibility with first-order} methods significantly boosts the trainability and breaks the scalability barrier for ONN training.
We now validate the above insights by extensive experiments.

\textbf{Training Efficiency Comparison with Prior Sparse Training Methods.}
In Figure~\ref{fig:MainResults}, we show accuracy and efficiency comparison of 1) baseline \name-SL (\texttt{BS}), 2) \name-SL with spatial sampling (\texttt{RAD}), 3) \name-SL with weight and spatial sampling (\texttt{SWAT-U}), and 4) \name-SL with all three introduced sampling methods (feedback, column, and data sampling), and 5) our proposed full flow with IC, PM, and sparse SL (\name).
To clarify, \name-SL performs subspace learning on-chip from scratch without using pre-trained weights, while \name includes the full flow, i.e., pre-training, mapping, and on-chip training.
When we perform subspace learning from scratch, our proposed \emph{multi-level sampling} strategies outperform previous \texttt{RAD} and \texttt{SWAT-U} by \textbf{$\sim$3$\times$} in hardware cost with comparable accuracy.
Though \texttt{RAD} can save the forward peak memory, it leaves the most expensive backward pass unoptimized, which does not fully exploit the sparsity in ONN training.
\texttt{SWAT-U} tries to save forward cost by simultaneously sparsifying the forward and feedback weights with shared masks/patterns.
However, in our experiment, the forward sparsification considerably degrades the model performance, which dilates the efficiency benefits from it.
Parallel mapping can fully leverage the pre-trained weights and help our full three-stage flow \name achieve the best accuracy with \emph{much faster convergence}, leading to \textbf{over 30$\times$} higher energy efficiency and fewer time steps.
\begin{figure}[h]
    \centering
    \includegraphics[width=\textwidth]{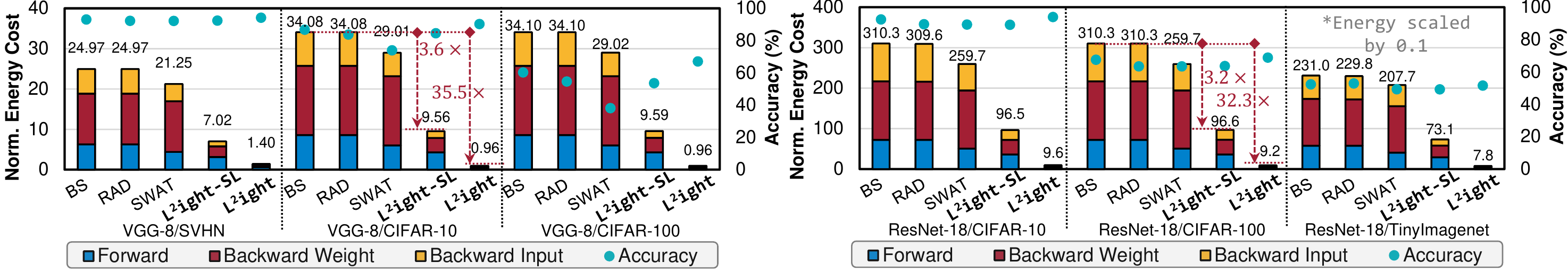}
    \caption{Accuracy and hardware efficiency comparison on VGG-8 (\emph{left}) and ResNet-18 (\emph{right}).}
    \label{fig:MainResults}
    \vspace{-5pt}
\end{figure}

Note that the energy efficiency and latency improvement is \emph{not just on the photonic part but a systematic performance boost}.
Our three-level sampling methods directly skip the pruned block, which means the corresponding cost of memory transaction, computation, control, and communication are removed together. 
Therefore, the sampling sparsity can be directly translated to the energy/latency improvement ratio regardless of whether the electrical part dominates the total cost.

\subsection{Ablation Studies and Discussion}
\label{sec:AblationStudy}
\subsubsection{Multi-Level Sparsity in Efficient Training}
\textbf{Feedback Sparsity.}
To investigate the impact of feedback sampling strategies, we visualize the gradient approximation fidelity and accuracy curves in Figure~\ref{fig:FeedbackAlgCompare}. \texttt{uniform} sampling shows varying performance under different sparsity values due to large gradient variances.
\texttt{topk} shows worse performance after sufficient steps due to its biased gradient estimation from overly greedy block selection.
In contrast, our proposed \textit{load-balancing} \texttt{btopk} strikes a balance between variance and bias via block-wise sampling and also leads to less runtime as it forces load balance among massively parallel PTCs.
In Table~\ref{tab:CompareBreakDown}, feedback sampling saves 50-60\% time steps on the most costly error feedback $\nabla_x\calL$, leading to 1.5-1.8$\times$ overall time step reduction with minimum accuracy drop.
\begin{figure*}[ht]
    \centering
    \vspace{-15pt}
    \subfloat[]{
    \includegraphics[width=0.3\textwidth]{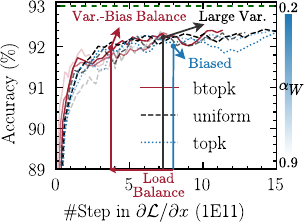}
    \label{fig:FeedbackAlgCompare}}
    \hspace{5pt}
    \subfloat[]{
    \includegraphics[width=0.3\textwidth]{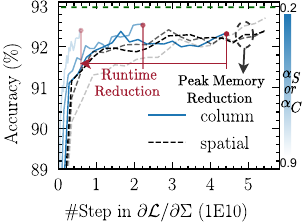}
    \label{fig:ConvFeatureAlgCompare}}
    \hspace{5pt}
    \subfloat[]{\includegraphics[width=0.275\textwidth]{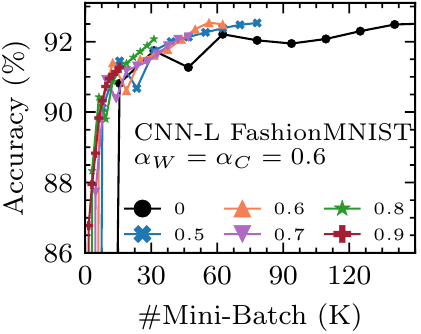}
    \label{fig:DataSamplingCompare}}
    \vspace{-5pt}
    \caption{~\small Accuracy v.s. weight gradient computation steps with three feedback sampling strategies (a) and different feature sampling techniques (b). 
    Accuracy (93.02\%) from a full-space trained model (green).
    CNN-L/FashionMNIST is used for (a) and (b).
    Compare different data sampling sparsity (c).}
\end{figure*}

\begin{table*}[b]
\centering
\caption{\small Compare sampling strategies on CIFAR-10 in terms of accuracy, activation size reduction, energy, and time step.
Forward, weight gradient, and error feedback are denoted as $\calL$, $\nabla_{\Sigma}\calL$, and $\nabla_{x}\calL$.
\name-SL is learning \emph{from scratch}, and \name (IC$\rightarrow$PM$\rightarrow$SL) is the full flow with pre-trained weights and non-ideal $\tilde{\bm{I}}$.}
\label{tab:VGG8SamplingCompare}
\resizebox{\textwidth}{!}{%
\begin{tabular}{|l|c|c|cccr|cccr|}
\hline
\multicolumn{1}{|c|}{\multirow{2}{*}{}} & \multirow{2}{*}{Acc$_{\pm\sigma}$ (\%)} & \multicolumn{1}{c|}{\multirow{2}{*}{Act$\downarrow$(\%)}} & \multicolumn{4}{c|}{Norm. PTC Energy}                    & \multicolumn{4}{c|}{Norm. \#Step}                       \\
\multicolumn{1}{|c|}{}                  &                          & \multicolumn{1}{c|}{}                         & $\calL$ & $\nabla_{\Sigma}\calL$ & $\nabla_{x}\calL$ & Total (Ratio) & $\calL$ & $\nabla_{\Sigma}\calL$ & $\nabla_{x}\calL$ & Total (Ratio)  \\ \hline\hline
\name-SL (Baseline) VGG-8                              & 86.66$_{\pm0.13}$                    & -                                            & 8.58    & 17.16           & 8.34           & 34.08 (1.00) & 32.64   & 5.49            & 92.02          & 130.14 (1.00) \\
+ Feedback Sampling ($\alpha_W$=0.6)                            & 86.41$_{\pm0.25}$                    & -                                            & 8.58    & 17.16           & 3.38           & 29.13 (1.17) & 32.64   & 5.49            & 35.76          & 73.89 (1.76)  \\
~~~~+ Column Sampling ($\alpha_C$=0.6)                         & 85.58$_{\pm0.01}$                    & -                                            & 8.58    & 7.16            & 3.38           & 19.12 (1.78) & 32.64   & 4.67            & 35.76          & 73.07 (1.78)  \\
~~~~~~~~+ Data Sampling ($\alpha_D$=0.5)                                & 84.45$_{\pm0.45}$                    & -                                            & 4.29    & 3.58            & 1.69           & 9.56 (3.56)  & 16.32   & 2.34            & 17.89          & 36.54 (3.56)  \\ \hline
+ RAD~\cite{NN_ICLR2021_Oktay} ($\alpha_S$=0.85)                                   & 83.68$_{\pm0.58}$                    &        11.78                  & 8.58    & 17.16           & 8.34           & 34.08 (1.00) & 32.64   & 5.49            & 92.02          & 130.14 (1.00) \\
+ SWAT-U~\cite{NN_NeurIPS2020_Aamir} ($\alpha_W$=0.3, $\alpha_S$=0.6)                                & 73.91$_{\pm0.27}$                    & 8.31                     & 6.01    & 17.16           & 5.84           & 29.01 (1.17) & 25.98   & 5.49            & 82.19          & 113.66 (1.15) \\ \hline
\textbf{\name (IC$\rightarrow$PM$\rightarrow$SL)}                                & \textbf{90.20$_{\pm0.05}$}                    & -                     & \textbf{0.43}    & \textbf{0.36}           & \textbf{0.17}           & \textbf{0.96 (35.64)} & \textbf{1.63}   & \textbf{0.23}            & \textbf{1.79}          & \textbf{3.65 (35.64)} \\\hline\hline
\name-SL (Baseline) ResNet-18      & 92.37$_{\pm0.08}$                    & -                                    & 72.24   & 144.49                 & 93.60             & 310.33 (1.00) & 463.40  & 27.23                  & 1,478.84          & 1,969.48 (1.00) \\
+ Feedback Sampling ($\alpha_W$=0.5)     & 91.35$_{\pm0.03}$                    & -                                    & 72.24   & 144.49                 & 48.13             & 264.86 (1.17) & 463.40  & 27.23                  & 747.22            & 1,237.85 (1.59) \\
~~~~+ Column Sampling ($\alpha_C$=0.5)       & 90.02$_{\pm0.16}$                     & 4.47                                 & 72.24   & 72.49                  & 48.13             & 192.86 (1.61) & 463.40  & 15.68                  & 747.21            & 1,226.30 (1.61) \\
~~~~~~~~+ Data Sampling ($\alpha_D$=0.5)         & 89.07$_{\pm0.04}$                    & 4.47                                 & 36.13   & 36.26                  & 24.07             & 96.46 (3.22)  & 231.76  & 7.84                   & 373.71            & 613.31 (3.21)   \\ \hline
+ RAD~\cite{NN_ICLR2021_Oktay} ($\alpha_S$=0.9)                   & 89.44$_{\pm0.17}$                    & 46.60                                 & 72.26   & 143.72                 & 93.60             & 309.58 (1.00) & 463.53  & 26.03                  & 1,478.84          & 1,969.00 (1.00) \\
+ SWAT-U~\cite{NN_NeurIPS2020_Aamir} ($\alpha_W$=0.3, $\alpha_S$=0.5) & 89.21$_{\pm0.16}$                     & 25.89                                & 50.57   & 143.64                 & 65.52             & 259.73 (1.19) & 358.40  & 26.56                  & 1,417.96          & 1,802.00 (1.09) \\ \hline
\textbf{\name (IC$\rightarrow$PM$\rightarrow$SL)}  & \textbf{93.91$_{\pm0.02}$}                    & 4.47                                 & \textbf{3.61}    & \textbf{3.62}                   & \textbf{2.41}              & \textbf{9.64 (32.20)}   & \textbf{23.16}   & \textbf{0.78}                   & \textbf{37.34}             & \textbf{61.29 (32.13)}    \\ \hline
\end{tabular}%
}
\label{tab:CompareBreakDown}
\vspace{-15pt}
\end{table*}

\textbf{Feature Sparsity.}~
Figure~\ref{fig:ConvFeatureAlgCompare} compares the accuracy and weight gradient computation time steps on two feature sampling techniques.
Though \textit{spatial sampling} (\texttt{ss}) can save peak storage by dropping a subset of activations during the forward pass, it shows no gradient computation step reduction. 
Our hardware-friendly \textit{column sampling} (\texttt{cs}) directly leads to energy and runtime reduction due to its structured sparsity.
In Table~\ref{tab:CompareBreakDown}, when column sampling is further added, we observe $\sim$50\% PTC energy saving on weight gradient computation $\nabla_{\Sigma}\calL$ at the cost of $\sim$1\% accuracy drop.

\textbf{Data Sparsity.}
In the data level, we also demonstrate how \texttt{SMD} with sparsity $\alpha_D$ impacts the training efficiency in Figure~\ref{fig:DataSamplingCompare}.
With the best selected $\alpha_W$ and $\alpha_C$,
data sparsity directly reduces training time by skipping iterations~\cite{NN_NeurIPS2019_Wang}.
The data sampling selects a uniform subset of the training set to represent the original data distribution, leading to less data processing with comparable generalization in the extracted features.
Another explanation is that the variance increased by partial replacement serves as a regularization mechanism to improve the model performance~\cite{NN_NeurIPS2019_Wang}.
For relatively easy tasks, aggressive sparsity ($\alpha_D$=0.8) is a sweet point, while for larger datasets shown in Table~\ref{tab:CompareBreakDown}, a medium sparsity (0.5) can be a good setting to balance both the training cost and accuracy.
With all three sampling methods integrated, our \name-SL shows competitive accuracy and $\sim$3$\times$ higher efficiency than \texttt{RAD} and \texttt{SWAT-U}.
More advanced dataset sampling methods are left for future exploration.

\subsubsection{Learnability of Restricted Subspace ONNs}
\textbf{Impacts of Calibration/Mapping Quality.}~
Table~\ref{tab:CompareBreakDown} shows that with IC and PM, the full \name flow %
\begin{wrapfigure}[12]{r}{0.35\textwidth}
\begin{minipage}[t]{0.35\textwidth}
\vspace{-10pt}
    \centering
    \includegraphics[width=0.97\textwidth]{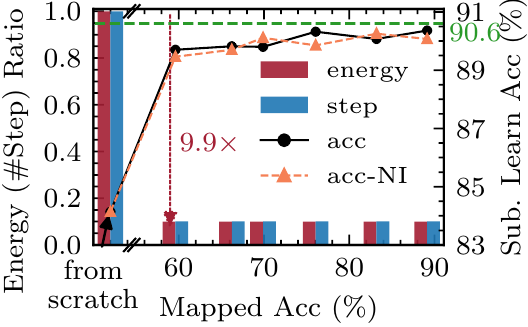}
    \vspace{-5pt}
    \caption{\small
    Impact of mapping accuracy
    (VGG-8 CIFAR-10 with $\alpha_W$=$\alpha_C$=0.6, $\alpha_D$=0.5).
    \textit{acc-NI} is the curve with non-ideal $\tilde{\bm{I}}$.}
    \label{fig:CompareDifferentMapAcc}
\end{minipage}
\end{wrapfigure}
achieves the highest accuracy with 32-35$\times$ efficiency boost over baselines.
We further evaluate the impact of different mapping accuracy and the calibration quality on subspace learning in Figure~\ref{fig:CompareDifferentMapAcc}.
First, \emph{parallel mapping or pre-training is not a must}.
Our subspace learning supports first-order optimization on-chip from random initialization.
Second, \emph{the optimality on subspace bases influences the final accuracy} as it determines the upper bound of accuracy that can be recovered by subspace learning.
With roughly optimized space bases, i.e., $\bm{U},\bm{V}^{\ast}$, subspace learning can efficiently train basis coefficients, i.e., $\bm{\Sigma}$, achieving 5-6\% higher accuracy and 9.9$\times$ less energy and steps compared with random unitaries (train from scratch).
Third, \emph{subspace optimization shows low sensitivity on mapping quality} and is able to compensate for the suboptimality in singular vectors within a reasonable range.
Even with 60\% mapped accuracy, singular value optimization has enough capability to recover the accuracy to $\sim$90\%.
Fourth, our subspace learning is \emph{robust to gradient noises} caused by non-ideal $\tilde{\bm{I}}$ ($MSE^U\!\!\approx$$MSE^V$$\approx$0.013), which shows that \name can tolerate reasonable suboptimality in the calibration and mapping stages.

\textbf{\emph{In-situ} Transferability in the Restricted Subspace.}
Another important question to answer is the %
\begin{wrapfigure}[14]{r}{0.61\textwidth}
\begin{minipage}[t]{0.61\textwidth}
\vspace{-15pt}
    \centering
    \subfloat[]{\includegraphics[width=0.49\textwidth]{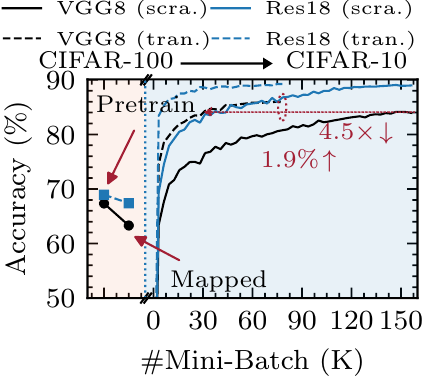}
    \label{fig:TransferVGG8}}
    \subfloat[]{\includegraphics[width=0.49\textwidth]{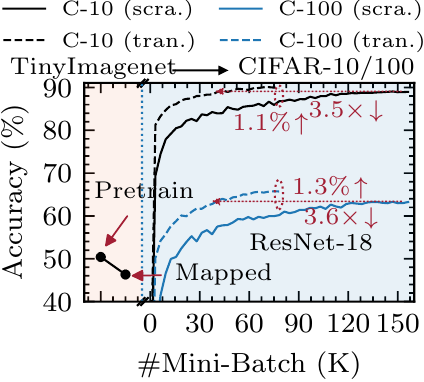}
    \label{fig:TransferResNet18}}
    \vspace{-8pt}
    \caption{~\small
    (a) Transfer VGG8/Res18 from CIFAR-100 to CIFAR-10.
    (b) Transfer Res18 from TinyImagenet to CIFAR-10 and 100.}
    \label{fig:TransferLearning}
\end{minipage}
\end{wrapfigure}
transferability of subspace learning.
After mapping, we fix the inherited unitaries and adapt to different tasks by only training the singular values.
Figure~\ref{fig:TransferLearning} shows that the inherited bases span a good design space with enough transferability.
The \textit{in-situ} subspace transfer learning shows 1-2\% higher final accuracy.
Also, it uses 3$\sim$5$\times$ fewer steps to obtain the same accuracy as training from scratch.
Hence, our proposed \name finds a highly trainable design point while the learnability is still mostly maintained.

\vspace{-.12in}
\section{Conclusion}
\label{sec:Conclusion}
In this work, we propose the \emph{first} scalable and efficient on-chip learning framework \name for emerging optical neural networks.
Our proposed three-stage flow synergistically enables on-chip self-learning via automatic circuit state calibration, parallel model mapping, and efficient subspace learning.
To further improve the learning efficiency, we explore multi-level sparsity, including balanced feedback sampling, information-preserving column feature sampling, and runtime-reduced data sampling.
Extensive ablation studies and comparison experiments show 3-order-of-magnitude scalability improvement over prior on-chip training protocols and 30$\times$ efficiency boost compared with previous sparse training methods.
In the future, we will go beyond current software simulation and experimentally validate the effectiveness of \name on real photonic neural chips.

\paragraph{Acknowledgments}

The authors acknowledge the Multidisciplinary University Research Initiative (MURI) program through the Air Force Office of Scientific Research (AFOSR), contract No. FA 9550-17-1-0071, monitored by Dr. Gernot S. Pomrenke.

\newpage
\appendix
\section{ONN Principles}
\label{sec:AppendixONNPrinciple}
\subsection{Mach-Zehnder Interferometers (MZIs)}
\label{sec:AppendixMZI}

A basic coherent optical component used in this work is an MZI.
One of the most general MZI structures is shown in Figure~\ref{fig:MZI}, consisting of two 50-by-50 optical directional couplers and four phase shifters $\theta_T$, $\theta_L$, $\omega_P$, and $\omega_W$.
An MZI can achieve arbitrary 2$\times$2 unitary matrices $SU(2)$.
\begin{figure}[h]
    \centering
    \includegraphics[width=0.3\textwidth]{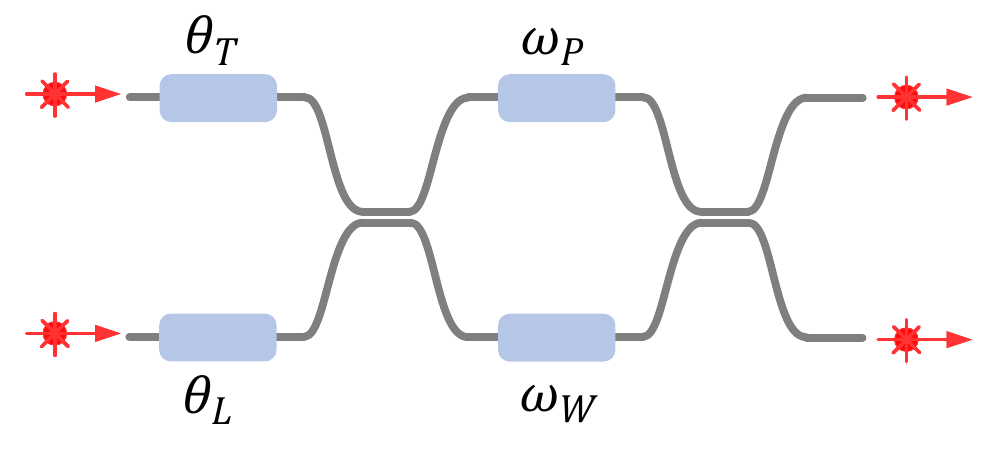}
    \caption{\small~2-by-2 MZI with top (T), left (L), upper (P), and lower (W) phase shifters.}
    \label{fig:MZI}
\end{figure}
The physical transfer matrix $R(\theta_g,\Delta\theta,\Delta\omega)$ of an MZI shown in Fig.~\ref{fig:MZI} is,
\begin{equation}
\small
\label{eq:MZITransferMatrix}
\begin{aligned}
    SU(2)=R(\theta_g,\Delta\theta,\Delta\omega)=&\begin{pmatrix}
       t & kj\\
       kj & t
    \end{pmatrix}
    \begin{pmatrix}
       e^{j\omega_P} & 0\\
       0 & e^{j\omega_W}
    \end{pmatrix}
    \begin{pmatrix}
       t & kj\\
       kj & t
    \end{pmatrix}
    \begin{pmatrix}
       e^{j\theta_T} & 0\\
       0 & e^{j\theta_L}
    \end{pmatrix}\\
    &=e^{j\theta_g}\begin{pmatrix}
    \sin{\frac{\Delta\omega}{2}} & \cos{\frac{\Delta\omega}{2}}\\
    \cos{\frac{\Delta\omega}{2}} & -\sin{\frac{\Delta\omega}{2}}
    \end{pmatrix}
    \begin{pmatrix}
    e^{j\frac{\Delta\theta}{2}} & 0\\
    0 & e^{-j\frac{\Delta\theta}{2}}
    \end{pmatrix},\\
    &\theta_g=\bar{\theta}+\bar{\omega}+\frac{\pi}{2},~\bar{\theta}=\frac{\theta_T+\theta_L}{2},~\bar{\omega}=\frac{\omega_P+\omega_W}{2},\\
    &\Delta\theta=\theta_T-\theta_L,~\Delta\omega=\omega_P-\omega_W,~ t=k=\frac{\sqrt{2}}{2}.
\end{aligned}
\end{equation}
where the global phase $\theta_g$ is determined by the common mode $\bar{\theta}$ and $\bar{\omega}$, and the light splitting is determined by the differential mode $\Delta\theta$ and $\Delta\omega$.
To achieve the 2-D planar rotator $R(2)$ in the real space parametrized by $\phi$, we let $\theta_T=\pi/2$, $\theta_L=3\pi/2$, $\bar{\omega}=\pi$.
To convert the simplified transfer matrix $M(\Delta\omega)$ to the planar rotator, we set $\Delta\omega=\pi-2\phi$ as follows,
\begin{equation}
\small
\label{eq:MZITransferMatrixSimplified}
\begin{aligned}
    R(2)=&e^{j\frac{3\pi}{2}}\begin{pmatrix}
    \sin{\frac{\Delta\omega}{2}} & \cos{\frac{\Delta\omega}{2}}\\
    \cos{\frac{\Delta\omega}{2}} & -\sin{\frac{\Delta\omega}{2}}
    \end{pmatrix}
    \begin{pmatrix}
    j & 0\\
    0 & -j
    \end{pmatrix}\\
    =&\begin{pmatrix}
    \sin{(\frac{\pi-2\phi}{2})} & \!\!-\cos{(\frac{\pi-2\phi}{2})}\\
    \cos{(\frac{\pi-2\phi}{2})} & \!\!\sin{(\frac{\pi-2\phi}{2})}
    \end{pmatrix}
    =\begin{pmatrix}
    \cos{\phi} & \!\!-\sin{\phi}\\
    \sin{\phi} & \!\!\cos{\phi}
    \end{pmatrix}.
\end{aligned}
\end{equation}

\subsection{MZI-based Photonic Tensor Core Architecture}
\label{sec:AppendixPTC}
By cascading $N(N-1)/2$ MZIs into a triangular mesh (Recks-style) or rectangular mesh (Clements-style), we can construct arbitrary $N\times N$ unitary $U(N)$.

As a simple example, we show the principle of Recks-style MZI array for a simple demonstration.
A similar decomposition can be derived for the Clements style.
It decomposes an $M \times N$ weight matrix using SVD, i.e., $\bm{W}=\bm{U\Sigma V^{\ast}}$.
The diagonal matrix $\bm{\Sigma}$ can be simply implemented by on-chip attenuators, e.g., single-port MZIs, to perform signal scaling.
The unitary matrices $\bm{U}$ and $\bm{V}^{\ast}$ can be realized by a cascaded MZI triangular array~\cite{NP_PHYSICAL1994_Reck}.
The unitary group parametrization is given by,
\begin{equation}
\small
\label{eq:UnitaryParametrization}
\bm{U}(N)=\bm{D}\;\prod_{i=N}^{2}\prod_{j=1}^{i-1}\bm{R}_{ij}(\phi_{ij}),
\end{equation}
where $\bm{D}$ is a diagonal matrix with $\pm{1}$ on its diagonal entries, and the 2-dimensional planar rotator $\bm{R}_{ij}(\phi_{ij})$ is an $n$-dimensional identity matrix where entries on ($i$,$i$), ($i$,$j$), ($j$,$i$), ($j$,$i$) are $\cos{\phi_{ij}}$, -$\sin{\phi_{ij}}$, $\sin{\phi_{ij}}$, $\cos{\phi_{ij}}$, respectively.
Each rotator $\bm{R}_{ij}$ can be implemented by a 2$\times$2 MZI that produces unitary interference of input light signals with a rotation angle $\phi$ as we show before.

\subsection{Optical Circuit Non-ideality}
\label{sec:AppendixNonideality}
\textbf{Rotation Quantization.}
Given the control resolution limits, we can only achieve discretized MZI rotation phase configurations.
We assume the phases $\phi$ is uniformly quantized into $b$-bit within [0,$2\pi$],
\begin{equation}
    \small
    \mathcal{Q}(\phi)=\texttt{Round}\Big(\frac{\phi~\text{mod}~2\pi}{2\pi/(2^{b}-1)}\Big)\frac{2\pi}{2^{b}-1}.
\end{equation}
We assume 8-bit quantization for phases of $\bm{U}$ and $\bm{V}^{\ast}$.
For $\bm{\Sigma}$ matrices, we assume larger bitwidths can be affordable and practical.

\textbf{Phase shifter Variation.}
Due to manufacturing error and thermal noises, the phase shift $\phi$ caused by a phase shifter is proportional to the device-related parameter, $\phi\propto\gamma$.
Assume the real coefficient drifts from the theoretical value $\gamma$ by $\Delta\gamma$, the real phase shift will become $\tilde{\phi}=\frac{\gamma+\Delta\gamma}{\gamma}\phi$.
We assume $\Delta\gamma\sim\mathcal{N}(0,0.002^2)$.
We denote this multiplicative error for all phase shifters as a diagonal $\bm{\Gamma}$ matrix, such that the non-ideal phase shifts become $\bm{\Phi}^v=\bm{\Gamma\Phi}$.

\textbf{MZI Crosstalk.}
Due to signal crosstalk, adjacent MZIs will have mutual coupling effects, such that the part of the phase shift $\phi$ for the $i$-th MZI will partially contribute to its neighboring MZI $\phi_j$ with a factor of $\omega_{i,j}$.
This crosstalk effect can be simply modeled as coupling matrix $\bm{\Omega}$,
{\small\begin{align}
    \small\centering
    \label{eq:Crosstalk}
    \begin{pmatrix}
    \phi^c_0\\
    \phi^c_1\\
    \vdots\\
    \phi^c_{N-1}
    \end{pmatrix}=&
    \begin{pmatrix}
    \omega_{0,0} & \omega_{0,1} & \cdots & \omega_{0,N-1} \\
    \omega_{1,0} & \omega_{1,1} & \cdots & \omega_{1,N-1} \\
    \vdots & \vdots & \ddots & \vdots \\
    \omega_{N-1,0} & \omega_{N-1,1} & \cdots & \omega_{N-1,N-1}
    \end{pmatrix}
    \begin{pmatrix}
    \phi^v_0\\
    \phi^v_1\\
    \vdots\\
    \phi^v_{N-1}
    \end{pmatrix}\notag\\
    \rm{s.t.}~~&\omega_{i,j}=1, \quad \forall\;i=j\notag\\
    &\omega_{i,j}=0, \quad \forall\;i\neq j \text{ and }\phi_j \in \mathcal{P}\\
    &0\leq\omega_{i,j}<1, \quad \forall\;i\neq j \text{ and }\phi_j \in \mathcal{A}\notag.
\end{align}}
The diagonal factor $\omega_{i,j}, i=j$ is the self-coupling coefficient.
$\omega_{i,j}, i\neq j$ is the mutual coupling coefficient~\cite{NP_JLT2019_milanizadeh, NP_DAC2020_Gu, NP_AAAI2021_Gu}.
We assume the self-coupling coefficient to be 1, and the mutual coupling coefficient is 0.005 for adjacent MZIs.

\section{Intractable Gradients for MZI Rotations}
\label{sec:AppendixPhaseGradient}
To optimize the MZI meshes, a straightforward idea is to use first-order methods to optimize all rotations phases $\bm{\Phi}^{U}$, $\bm{\Phi}^{V}$, and $\bm{\Phi}^{\Sigma}$.
The analytical gradients for phases in unitary matrices are shown as,
\begin{equation}
\small
\begin{aligned}
    \label{eq:GradientUnitaryParametrization}
    \frac{\partial\mathcal{L}}{\partial\bm{R}_{ij}}&=\big(\bm{D}\bm{R}_{n1}\bm{R}_{n2}\bm{R}_{n3}\big)^T\nabla_{y}\mathcal{L}~x^{T}\big(\cdots\bm{R}_{32}\bm{R}_{21}\bm{\Sigma}\bm{V}^{\ast}\big)^T\\
    \frac{\partial\mathcal{L}}{\partial\phi_{ij}}&=\mathrm{Tr}\bigg(\Big(\frac{\partial\mathcal{L}}{\partial\bm{R}_{ij}}\odot\frac{\partial{\bm{R}_{ij}}}{\partial{\phi_{ij}}}\Big)(e_i+e_j)(e_i+e_j)^T\bigg).
\end{aligned}
\end{equation}
Therefore, it is prohibitively expensive to derive the analytical phase gradients, which is one of the key motivations for our subspace optimization method.

\section{Detailed Description of the Proposed Parallel Mapping Algorithm}
\label{sec:AppendixMapAlg}
We give a detailed description of our parallel mapping algorithm.
Zeroth-order coordinate descent (\texttt{ZCD}) is used as an example.
In line 4, we first derive and implement the optimal theoretical singular values and initialize $\bm{\Phi}^U$ and $\bm{\Phi}^V$ using the decomposed values.
In lines 8-13, we use \texttt{ZCD} to alternately optimize phases in $\bm{U}$ and $\bm{V}^{\ast}$ under all non-ideal effects till convergence.
The step size is strictly bounded by the smallest phase control resolution.
Exponential decay is used to quickly reduce the learning rate to avoid divergence.
Note that cosine-annealing will not work since the ZO descent will rapidly converge given its greedy search nature.
Then at the end, due to the suboptimality in \texttt{ZCD}, we will perform \texttt{OSP} to find the current optimal singular values that minimize the mapping error given the trained $\bm{U}^T$ and $\bm{V}^{\ast,T}$.

\begin{algorithm2e}[H]
    \SetAlgoLined
    \SetKwInOut{Input}{Input}
    \SetKwInOut{Output}{Output}

    \Input{Mapping loss $\mathcal{L}^M$, mapping target $\bm{W}$, total iterations $T$, inner ZCD iterations $S$, step size decay factor $\beta$, ZCD step size upper bound $\delta\phi_{u}=\frac{2\pi}{2^{\min(b_{l},b)}-1}$, ZCD step size lower bound $\delta\phi_{l}=\frac{2\pi}{2^{\min(b_m,b)}-1}$}
    $\delta\phi=\delta\phi_u$\;
    \For{\text{Weight block} $\bm{W}_{pq}\sim\bm{W}$}{
    \text{Step 1: SVD and Parametrization via Eq.~\eqref{eq:Objective}}\;
    $\bm{U}_{pq}(\bm{\Phi}_{pq}^U), \bm{\Sigma}_{pq}(\bm{\Phi}_{pq}^S), \bm{V}^*_{pq}(\bm{\Phi}_{pq}^V)=\texttt{UP}\big(\texttt{SVD}(\bm{W}_{pq})\big)$\;
    \text{Step 2: ZCD on }$\bm{U}_{pq}, \bm{V}^*_{pq}$\;
        \For{$t\gets 0\cdots T-1$}{
            \For{$s\gets 0\cdots S-1$}{
                \text{Randomly sample a phase $\phi\sim\{\bm{\Phi}^U_{pq},\bm{\Phi}^V_{pq}\}$}\;
                \If{$\mathcal{L}^M_{pq}(\phi^{tS+s}+\delta\phi)<\mathcal{L}^M_{pq}(\phi^{tS+s})$}{
                $\phi^{tS+s+1}\gets\phi^{tS+s}+\delta\phi$\;
                }
                \Else{
                $\phi^{tS+s+1}\gets\phi^{tS+s}-\delta\phi$\;
                }
                $\delta\phi\gets\max(\delta\phi/\beta,\delta\phi_l)$\;
            }
        }
        Step 3: Optimal Projection on $\bm{\Sigma}_{pq}$\;
            $\Sigma_{pq}\gets\texttt{diag}(\tilde{\bm{I}}^{\ast}\bm{U}^*_{pq}\bm{W}_{pq}\bm{V}_{pq}\tilde{\bm{I}})$\;
    }
    \Output{Converged phases $\bm{\Phi}^{M}$}
    \caption{Parallel Mapping with \texttt{ZCD} and \texttt{OSP}}
    \label{alg:ParallelMapping}
\end{algorithm2e}

\section{Prove of Unbiased Gradient Approximation with Feedback and Feature Sampling}
\label{sec:AppendixUnbiasedGradient}

\begin{myclaim}
Considering the $l$-th layer with input $x\in \mathbb{R}^N$ and pre-activation $y\in \mathbb{R}^M$, we denote the blocking weight matrix as $\bm{W}=\{\bm{W}_{pq}\}_{p,q=1,1}^{P=\frac{M}{k},Q=\frac{N}{k}}$ and nonlinear activation as $\sigma$.
During backward, we randomly sample the feedback matrix $\bm{W}^T\in\mathbb{R}^{N\times M}$ with a structured sparse mask $\calP_{\bm{W}}=c_W(\calS_W\otimes\bm{1})$.
A similar sampling matrix $\calP_{x}$ is applied to input features.
The estimated gradients are unbiased, i.e., $\mathbb{E}[\big(\frac{\partial\mathcal{L}}{\partial\bm{\Sigma}}\big)_{\mathcal{S}}]=\frac{\partial\mathcal{L}}{\partial\bm{\Sigma}}$.

\end{myclaim}
\begin{proof}
Given $\mathbb{E}[\mathcal{P}]=\bm{1}$, we have
\begin{equation}
    \small
    \label{eq:UnbiasedWeight}
    \begin{aligned}
        \mathbb{E}[(\bm{W}_l^T)_{\calS_{\bm{W}_l}}]=&\mathbb{E}[\bm{W}_l^T\odot\calP_{\bm{W}_l}]=\bm{W}_l^T\\
        \mathbb{E}[(\bm{x}_l^T)_{\calS_{\bm{x}_l}}]=&\mathbb{E}[\bm{x}_l^T\odot\calP_{\bm{x}_l}]=\bm{x}_l^T.
    \end{aligned}
\end{equation}
Then we can derive
\begin{equation}
    \small
    \label{eq:UnbiasedGrad}
    \begin{aligned}
        \mathbb{E}[\big(\pfrac{\calL}{ y_l}\big)_{\calS_{\bm{W}_l}}]=&\mathbb{E}\Big[\sigma_l'\prod_{i=l+1}^{L-1}((\bm{W}_i^T)_{\calS_{\bm{W}_l}}\odot\sigma_i')(\bm{W}_L^T)_{\calS_{\bm{W}_l}}\pfrac{\calL}{y_L}\Big]
        =\pfrac{\calL}{y_l}\\
        \mathbb{E}\big[\big(\pfrac{ \calL}{{\bm{\Sigma}}_l}\big)_{\mathcal{S}}\big]=&\mathbb{E}\Big[\bm{U}^{\ast}\big(\pfrac{\calL}{ y_l}\big)_{\calS_{\bm{W}_l}}(x_l^T)_{\calS_{\bm{x}_l}}\bm{V}\Big]
        =\pfrac{\calL}{\bm{\Sigma}_l}.
    \end{aligned}
\end{equation}
\end{proof}

\section{Training Details}
\label{sec:AppendixTrainDetails}
We implement ONN simulation, all models, and training logic in PyTorch 1.8.1.
All experiments are conducted on a machine with an Intel Core i7-9700 CPU and an NVIDIA Quadro RTX 6000 GPU.
For identity calibration, we set the epoch to 400 with an initial learning rate of 0.1, a decay rate of 0.99, and a phase resolution of 8 bit.
For parallel mapping, we set the epoch to 300 with an initial learning rate of 0.1, a decay rate of 0.99, and a phase resolution of 8 bit.
For subspace learning, we adopt AdamW as the optimizer with a learning rate of 0.002 and a weight decay rate of 0.01 for subspace learning from scratch.
Epochs are set to 100 for MNIST, FashionMNIST training, 200 for CIFAR-10/100, and TinyImageNet.
For subspace learning after mapping, we reduce the epoch to 20 and the learning rate to 0.0002.
We use cosine-annealing as the learning rate scheduler.
When compared with prior on-chip learning protocols, we adopt the recommended settings for \texttt{FLOPS} and \texttt{MixedTrn} in~\cite{NP_DATE2020_Gu, NP_AAAI2021_Gu}.
For \texttt{FLOPS}, the total epochs are set to 50, the initial learning rate is 2, and the gradient samples are set to 5.
For \texttt{MixedTrn}, we train for 20 epochs, the mixed-training sparsity is set to 0.4, the parameter sparsity is set to 0.1, and the initial learning rate is set to 0.02.
When compared with prior sampling methods, we apply uniform spatial sampling with expectation-maintained normalization for \texttt{RAD}~\cite{NN_ICLR2021_Oktay}.
For \texttt{SWAT-U}~\cite{NN_NeurIPS2020_Aamir}, we apply uniform spatial feature sampling without normalization and uniform weight matrix sampling with expectation-maintained normalization.
Since we only perform efficient training, we turn off any sampling in inference.

\section{MZI Array Scaling}
\label{sec:AppendixScaling}
A single MZI array has a limited size due to its high area cost, e.g., up to 32 or 64.
However, this is not an issue for our framework.
Multi-core systems with small subarrays are trends for analog computing, which is the design concept of our accelerator in Figure~\ref{fig:ONNTrainArch}.
Multiple PTCs are interconnected to support a large tensor computation in parallel. Therefore, our system’s performance will not be limited by the scale of a single PTC.
Actually, partitioning a large tensor operation into small chunks is widely adopted and recently considered as a better solution than large array sizes due to noise robustness consideration.

We adopt 9$\times$9 blocks based on the following considerations.

\noindent\textbf{Hardware practicality.}~ The largest commercial demonstration of optical neural chips is 32$\times$32 so far.
9$\times$9 is a practical, robust, and efficient setting according to recent experimental demonstrations.

\noindent\textbf{Robustness.}~ Larger MZI arrays will cause severe phase error accumulation effects.
Cascaded phase %
\begin{wrapfigure}[7]{rH}{0.52\textwidth}
\begin{minipage}{0.52\textwidth}
\resizebox{\textwidth}{!}{
\begin{tabular}{c|cccccc}
\toprule
Blk size  & 8     & 9     & 12    & 16    & 24    & 32    \\ \midrule
Rel. Err. & 0.025 & 0.032 & 0.043 & 0.061 & 0.094 & 0.126 \\
std.      & 2e-4  & 3e-4  & 3e-4  & 5e-4  & 9e-4  & 1e-3  \\ \bottomrule
\end{tabular}
}
\captionof{table}{Relative matrix error with different MZI array sizes.}
\end{minipage}
\end{wrapfigure}
error will cause non-trivial fidelity and robustness issues as block size increases.
9$\times$9 is generally a robust design configuration when cascaded noises are still tolerable.
Here we show a table of noise-induced errors (relative matrix distance) with various block sizes on a 256$\times$256 weight matrix.
Std. is calculated based on 20 runs.
Phase shifter gamma noise std=0.002, crosstalk factor=0.005, quantization bitwdith=8-bit.
We observe large array sizes are noise-sensitive in general.

\noindent\textbf{ZOO Convergence.}~
IC and PM are zeroth-order optimization techniques.
Each block indicates an optimization instance.
A larger block size  will have negative impacts on the optimization convergence %
\begin{wrapfigure}[5]{rH}{0.55\textwidth}
\begin{minipage}{0.55\textwidth}
\resizebox{\textwidth}{!}{
\begin{tabular}{c|cccccc}
\toprule
Blk size  & 8     & 9     & 12    & 16    & 24    & 32    \\ \midrule
($MSE^U+MSE^V$)/2 & 0.0135 & 0.013 & 0.03 & 0.039 & 0.04 & 0.045 \\ \bottomrule
\end{tabular}
}
\captionof{table}{IC optimality with different array sizes.}
\end{minipage}
\end{wrapfigure}
and solution optimality, which is the intrinsic limitation of most zeroth-order optimizers.
In the IC procedure, for relatively large block sizes, our ZO optimizers, unfortunately, will have solution quality degradation due to the curse of dimensionality and efficiency degradation due to low parallelism.
Here we show how solution quality in identity calibration changes with various block sizes.
9$\times$9 block is a good selection with high solution quality.

\noindent\textbf{Parameter Space.}~
Subspace learning only optimizes the singular values while $U$ and $V$ are fixed.
For an $N\times N$ weight matrix with $k\times k$ blocks, only $N^2/k$ singular values are trainable.
Increasing %
\begin{wrapfigure}[7]{rH}{0.55\textwidth}
\begin{minipage}{0.55\textwidth}
\resizebox{\textwidth}{!}{
\begin{tabular}{c|cccccc}
\toprule
Blk size  & 8     & 9     & 12    & 16    & 24    & 32    \\ \midrule
Accuracy & 84.26 & 84.45 & 83.36 & 81.27 & 80.68 & 78.40 \\ \bottomrule
\end{tabular}
}
\captionof{table}{Subspace learning accuracy with different block sizes.}
\end{minipage}
\end{wrapfigure}
the block size $k$ will decrease the parameter space.
According to the experience from the field of structured/subspace neural networks, e.g., block-circulant neural nets, the block size is typically set to a number around 8.
Here we add new results on \name-SL ($\alpha_W$=$\alpha_C$=0.6, $\alpha_D$=0.5) CIFAR-10 VGG8 with various block sizes.
According to our experiments below, 16$\times$16 blocks already show inadequate trainability due to overly small parameter space, leading to a clear accuracy drop.
In conclusion, we recommend using multiple interconnected 9$\times$9 PTCs for parallel computing, since this choice of 9$\times$9 block balances both systematic performance, hardware complexity, robustness, and on-chip trainability.

\section{Hardware Cost Evaluation}
\label{sec:AppendixCost}
\subsection{PTC Energy Estimation}
For simplicity, we count the number of PTC calls as the indicator to the total energy estimation of the PTC cluster.
For example, we focus on a 2-D convolutional layer with kernel shape of $C_{out}\times C_{in} \times K\times K$, input feature size $B\times C_{in} \times H\times W$ output feature size of $B\times C_{out} \times H'\times W'$.
We partition the unfolded weight matrix into $P\times Q$ blocks with size of $k \times k$ and assign each to a PTC.
We have $P=\ceil{\frac{C_{out}}{k}}$ and $Q=\ceil{\frac{C_{in}\times K^2}{k}}$.
Each PTC can utilize $k$ wavelengths to achieve parallel processing.
Now we give detailed computation of energy breakdown per optimization iteration.
\begin{equation}
    \small
    \label{eq:EnergyPerStep}
    \begin{aligned}
    \textbf{Forward Energy}&=C_{out}C_{in}K^2BH'W'\\
    \textbf{Backward Weight Energy}&=2\text{Tr}(\calS_C^T\calS_C)BPQ\\
    \textbf{Backward Input Energy}&=\text{Tr}(\calS_W^T\calS_W)BHW.
    \end{aligned}
\end{equation}
Note that in backward weight energy, we double the PTC call since the \emph{in-situ} subspace gradient acquisition requires 2 PTC calls.

\subsection{Total Time Step Estimation}
We assign $k$ electrical adders for each PTC to implement sequential cross-PTC reduction and parallel local accumulation.
Each PTC call counts as one step, each partial product/gradient accumulation stage counts as one step, and the Hadamard multiplication in gradient computation also counts as one step.
Given this assumption, we derive the time step as,
\begin{equation}
    \small
    \label{eq:TimestepPerStep}
    \begin{aligned}
    \textbf{Forward Step}&=(Q-1)_+BH'W'+\ceil{\frac{BH'W'}{k}}\\
    \textbf{Backward Weight Step}&=4\text{Tr}(\calS_C^T\calS_C)B\\
    \textbf{Backward Input Step}&=\left\{
    \begin{aligned}
        \ceil{\frac{C_{in}}{P}}\ceil{\log_2{2k}}\ceil{\frac{1}{2}\max_q\Big(\big(\sum\calS_W(q,:)-1\big)_+\Big)}BHW,&\quad K>1,~\text{stride}<K\\
        \max_q\Big(\big(\sum\calS_W(q,:)-1\big)_+\Big)BH'W',&\quad K=1
    \end{aligned}\right.
    \end{aligned}
\end{equation}

\subsection{WDM Dispersion Discussion}
\label{sec:AppendixWDMDIspersion}

Theoretically, coherent photonic circuits will have slightly different phase responses to different working wavelengths.
However, we claim that this frequency-specific phase shift has minimum impacts on our learning procedure.

\noindent\textbf{Negligible Dispersion.}~ Our PTC core is intentionally designed to have a small-scale, i.e., 9$\times$9. Hence we require 9 wavelengths in our framework.
This avoids too many wavelengths being used.
Therefore, the spectrum range will be relatively small.
Conservatively we assume 8 nm between the furthest two wavelengths.
Based on the phase response equation, $\Delta\phi(\lambda)=2\pi n_{eff}(\lambda)L/\lambda$, this leads to a maximum 1-2\% phase difference for the furthest two wavelengths.
On a small MZI array, this phase difference will only cause negligible transfer function drift.
We simulate this effect when the weight block size is set to 9$\times$9 and inject 1-2\% dispersion-induced MZI phase response drift;
the transfer matrix has ~0.5\% relative error and ~0.5\% mean square error.
Compared with the gradient approximation error caused by our three-level sparse sampling, phase variation, and thermal crosstalk, shown in Fig. 8, this slight drift caused by WDM dispersion is negligible.

\noindent\textbf{High Non-ideality Tolerance.}~ Our experiments show that first-order subspace learning is very robust to all these gradient approximation errors.
With all the above non-ideality, the approximated gradient directions are still well-aligned with the true gradients.
The on-chip learning procedure works as expected even when WDM dispersion effects are considered.
This effect can be considered in-situ when using WDM on MZI array training, therefore, the model can tolerate this non-ideal effect without inference accuracy degradation.

\noindent\textbf{Dispersion-free Devices.}~~In the literature, there are WDM dispersion-free MZI devices being proposed~\cite{NP_JLT2015_Dupuis}.
Within the 45nm range, the coefficient of phase shifters can be maintained.
Thus, the phase response to 9 different wavelengths can be compensated to almost the same response.
This further shows that WDM dispersion is not a major concern for our assumed ONN architecture and proposed training flow.

\end{document}